\documentclass[twoside,11pt]{article}

\usepackage[hyperref]{jmlr2e_mod}

\usepackage{amssymb}
\usepackage{amsmath}
\usepackage{graphicx}
\usepackage{subfigure}
\usepackage[sectionbib]{bibunits}
\usepackage{makeidx}
\usepackage{multicol}
\usepackage{url}

\newcommand{\R}{\mathbb{R}}
\newcommand{\Rbar}{\overline{\mathbb{R}}}

\newcommand{\pmat}[1]{\begin{pmatrix}#1\end{pmatrix}}
\newcommand{\lambdaSum}{\lambda_\text{sum}}
\newcommand{\lambdaMax}{\lambda_\text{max}}
\newcommand{\lambdaL}{\lambda_\text{L}}
\newcommand{\lambdaS}{\lambda_\text{S}}
\newcommand{\tauSum}{\tau_\text{sum}}
\newcommand{\tauMax}{\tau_\text{max}}

\DeclareMathOperator{\proj}{Proj}
\newcommand{\Proj}{\proj}
\newcommand{\dom}{\mathrm{dom}}

\newcommand{\Linear}{\mathcal{L}}
\newcommand{\AAA}{\Linear}

\newcommand{\data}{A} 

\newcommand{\obs}{\data} 

\newcommand{\phiSum}{\phi_{\text{sum}}}
\newcommand{\phiMax}{\phi_{\text{max}}}
\newcommand{\gauge}[2]{\gamma\left(#1 \mid #2 \right)}

\newcommand{\smoothFcn}{\omega}
\newcommand{\proxFcn}{\psi}
\newcommand{\linear}{\Linear}
\newcommand{\offset}{b}

\newcommand{\bigProxFcn}{\Phi}
\newcommand{\bigDualFcn}{\Psi}
\newcommand{\bigX}{\mathbf{z}} 
\newcommand{\bigY}{\mathbf{w}}
\newcommand{\dual}{\bigX}
\newcommand{\smalldual}{z}
\newcommand{\dualFcn}{q}
\newcommand{\bigLinear}{\mathbf{L}}

\newcommand{\Lipschitz}{\ell}

\newcommand{\iprod}[2]{\left\langle #1,\, #2\right\rangle}
\DeclareMathOperator*{\argmin}{arg\,min}

\DeclareMathOperator{\prox}{Prox} 

\newcommand{\order}{\mathcal{O}}
\DeclareMathOperator{\diag}{diag}

\newtheorem{assumption}{Assumption}
\newcommand{\defeq}{\stackrel{\text{\tiny def}}{=}}  

\newcommand{\vertiii}[1]{{\left\vert\kern-0.25ex\left\vert\kern-0.25ex\left\vert #1
    \right\vert\kern-0.25ex\right\vert\kern-0.25ex\right\vert}}

\newcommand{\matrixNorm}[1]{\vertiii{#1}} 

\newcommand{\Loracle}{L_\text{oracle}}
\newcommand{\Soracle}{S_\text{oracle}}

\newcommand{\norm}[1]{\| #1 \|}

\newcommand{\<}{\langle}
\renewcommand{\>}{\rangle}

\newcommand{\eps}{\varepsilon}

\usepackage[usenames,dvipsnames,svgnames]{xcolor}

\usepackage{algorithm}
\usepackage{algorithmic}
\usepackage{dsfont}
\usepackage{times}
\usepackage{graphicx}

\usepackage{amsmath}
\usepackage{amssymb}
\usepackage{pgfplots}

\usepackage{subdepth}
\usepackage{microtype}
\usepackage{hyperref}
\hypersetup{pdfauthor={Aleksandr Aravkin, Stephen Becker},
    pdftitle={Dual Smoothing and Level Set Techniques for Variational Matrix Decomposition},
    colorlinks=true,
    citecolor=MidnightBlue,
    urlcolor=Bittersweet,
    bookmarks=true,
    bookmarksopen,
    bookmarksopenlevel=2,
}

\ShortHeadings{Dual Smoothing and Level Set Techniques for Variational Matrix Decomposition}{A. Aravkin and S. Becker}
\firstpageno{1}

\begin{document}

\title{Dual Smoothing and Level Set Techniques for Variational Matrix Decomposition}

\author{\name Aleksandr Aravkin \email saravkin@uw.edu \\
       \addr Department of Applied Mathematics\\
       University of Washington\\
       Seattle, WA 
       \AND
       \name Stephen Becker \email Stephen.Becker@colorado.edu \\
       \addr Department of Applied Mathematics\\
       University of Colorado, Boulder\\
       Boulder, CO 
       }

\editor{chapter in ``Robust Low-Rank and Sparse Matrix Decomposition: Applications in Image and Video Processing'', T. Bouwmans, N. Aybat, E. Zahzah, eds. CRC Press, 2016.}
    
\maketitle

\begin{abstract}
 We focus on the robust principal component analysis (RPCA) problem, and review 
 a range of old and new convex formulations for the problem and its variants.
We then review dual smoothing and level set techniques in convex optimization, 
present several novel theoretical results, and apply the techniques 
on the RPCA problem. In the final sections, we show a range of numerical  experiments 
for simulated and real-world problems.   \end{abstract}

\section{Introduction}
Linear superposition is  a useful model for many applications, including nonlinear mixing problems. Surprisingly, we can perfectly distinguish multiple elements in a given signal using convex optimization as long as they are concise and look sufficiently different from one another. 
{Robust principal component analysis (RPCA)} is a 
key example, where we decompose a signal into the sum of low rank and sparse components.
RPCA is a special case of {\it stable principal component pursuit (SPCP)}\index{stable principal component pursuit}, 
where we also allow an explicit noise component within the RPCA decomposition. 
Applications include alignment of occluded images~\citep{PenGanWri:12}, 
scene triangulation~\citep{ZhaLiaGan:11},  model selection~\citep{CPW12}, face recognition, and document indexing~\citep{CanLiMa:11}. 

For SPCP, our model is
\begin{equation} \label{eqModel}
\obs = L + S + E
\end{equation}
where $A$ is the observed matrix, $L$ is a low-rank matrix, $S$ is a sparse matrix, and $E$ is an unstructured nuisance matrix (e.g., a stochastic error term).
The classic RPCA formulation \citep{CanLiMa:11} assumed $E=0$,  
but in general we do not distinguish between RPCA and SPCP. 

The RPCA problem uses regularization on the summands $L$ and $S$
in order to improve the recovery of the solution. In~\citep{candes2011robust},
the 1-norm regularizer is applied to $S$ to promote sparsity, and the nuclear norm
is applied to $L$ to penalize rank: 
\begin{equation}
\label{eq:classicRPCA}
\min \matrixNorm{L}_* + \lambda \|S\|_1 \quad \mbox{s.t.} \quad \obs = L + S.
\end{equation}
The 1-norm $\|\cdot\|_1$ and nuclear norm $\matrixNorm{\cdot}_*$ are given by
$
\|S\|_1 = \sum_{i,j} |s_{i,j}|$ and $\matrixNorm{L}_* = \sum_{i} \sigma_i(L),$ 
where $\sigma(L)$ is the vector of singular values of $L$.
The parameter $\lambda > 0$ controls the
relative importance of the low-rank term $L$ vs.\ the sparse term $S$.
This problem has been analyzed by~\citep{ChaSanPar:09,CanLiMa:11}, and it has perfect recovery guarantees under stylized incoherence assumptions. There is even theoretical guidance for selecting a minimax optimal regularization parameter $\lambda$ \citep{CanLiMa:11}. 

There are several modeling choices underlying formulation~\eqref{eq:classicRPCA}.
First is the choice of the $\ell_1$-norm to promote sparsity and the trace norm (aka 
nuclear norm) 
to promote low-rank solutions. 
We will keep with these choices throughout the entire chapter, noting where it is possible
to use more general penalties.
Second,~\eqref{eq:classicRPCA} assumes the data are fit exactly. 
Unfortunately, many practical problems only approximately satisfy the idealized assumptions.
This motivates the SPCP variant:
\begin{equation} \label{eq:sum-SPCP} \tag{$\text{SPCP}_\text{sum}$}
\begin{aligned}
& \min_{L,S}\; \matrixNorm{L}_* + \lambdaSum \|S\|_1 \\
& \text{subject to}\; \|L+S-A\|_F \le \eps, 
\end{aligned}
\end{equation}
where the  $\eps$ parameter accounts for the unknown perturbations
$A-(L+S)$ in the data not explained by $L$ and $S$. 
It is useful to define $\phi(L,S) = \matrixNorm{L}_* + \lambda \|S\|_1$ as 
a regularizer on the decision variable $(L,S)$. The formulation~\eqref{eq:classicRPCA} then tries to find 
the tuple $(\overline L, \overline S)$ that fits the data perfectly, and is minimal with respect 
to $\phi$. 

Third, the functional form of $\phi$ is important; in~\eqref{eq:classicRPCA} as well as~\eqref{eq:sum-SPCP}
the component penalties are added with a tradeoff parameter $\lambda$,
but other choices can be made as well. 
In particular,~\citet{aravkin2014variational} propose a new variant called ``max-SPCP'':
\begin{equation} \label{eq:max-SPCP} \tag{$\text{SPCP}_\text{max}$}
\begin{aligned}
& \min_{L,S}\; \max\left( \matrixNorm{L}_* , \lambdaMax \|S\|_1 \right) \\
& \text{subject to}\; \|L+S-A\|_F \le \eps,
\end{aligned}
\end{equation}
where $\lambdaMax>0$ acts similar to $\lambdaSum$, and 
this new formulation  offers both modeling and computational advantages over \eqref{eq:sum-SPCP}
(see Section~\ref{sec:numerics}). We show that cross-validation with \eqref{eq:max-SPCP} to estimate $(\lambdaMax,\eps)$ is significantly easier than estimating $(\lambdaSum,\eps)$ in \eqref{eq:sum-SPCP}. 
Given an \emph{oracle} that provides an ideal separation $A \simeq \Loracle + \Soracle$, 
we can use $\eps = \|\Loracle  + \Soracle - A \|_F$ in both cases. 
However, while we can estimate $\lambdaMax = \matrixNorm{\Loracle}_*/\|\Soracle\|_1$, 
it is not clear how to choose $\lambdaSum$ from data, without 
using cross-validation or assuming a probabilistic model.  

Finally, both~\eqref{eq:sum-SPCP} and~\eqref{eq:max-SPCP} assume
a least-squares penalty functional to measure the error up to level $\epsilon$. 
We can consider a more general choice of penalty $\rho$:
\begin{equation}
\label{eq:generalRPCA}
\min \phi(L,S) \quad \mbox{s.t.} \quad \rho(A - L - S) \leq \eps.
\end{equation}
Robust losses as well as $\rho$ arising from probabilistic models have been explored 
in~\citep{aravkin2014fast,aravkin2016level}.

Once $\rho$ and $\phi$ have been selected, we can choose the type of regularization
formulation one wants to solve. Formulation~\eqref{eq:generalRPCA} minimizes the regularizer
subject to a constraint on the misfit error. Two other common formulations are
\begin{equation}
\label{eq:conRPCA}
\min \rho(A - L - S) \quad \mbox{s.t.} \quad  \phi(L,S) \leq \tau,
\end{equation}
which minimizes the error subject to a constraint on the regularizer, and 
\begin{equation}
\label{eq:lagRPCA}
\min \rho(A - L - S) +\lambda  \phi(L,S),
\end{equation}
which minimizes the sum of error and regularizer with another tradeoff parameter 
to balance these goals. 

All three formulations can be effectively used, and are equivalent in the sense that solutions 
match for certain values of parameters $\lambda$, $\tau$, and $\eps$.
Formulation~\eqref{eq:generalRPCA} is preferable from a modeling perspective when 
the misfit level $\eps$ is known ahead of time, or can be estimated. 
However, formulations~\eqref{eq:conRPCA} and~\eqref{eq:lagRPCA} often have 
fast first-order algorithms available for their solution. 

It turns out that we can exploit algorithms for~\eqref{eq:conRPCA} to solve~\eqref{eq:generalRPCA}
using the graph of the value function for problem~\eqref{eq:conRPCA};
this relationship can be used to show that the problems have the same complexity~\citep{aravkin2016level}. 
Level set optimization was first applied for sparsity optimization by~\citet{SPGL}, and later extended
to gauge optimization~\citep{BergFriedlander:2011}
and to general convex programming~\citep{AravkinBurkeFriedlander:2013}. 

The second approach we consider is the TFOCS algorithm~\citep{TFOCS} and software\footnote{\url{http://cvxr.com/tfocs}}, which is based on the proximal point algorithm, and can also handle generic convex minimization problems. 
We present a new analysis of TFOCS, along with stronger convergence guarantees, and 
also apply TFOCS method to RPCA. TFOCS solves all standard variants of RPCA and SPCP, 
and can easily add non-negativity or other types of additional constraints. 
We briefly detail how the algorithm can be specialized for the RPCA problem in particular.

The chapter proceeds as follows. In Section~\ref{section-analysis}, we provide the necessary convex analysis background to understand our algorithms and results. 
In Section~\ref{sec:variational}, we look at level set techniques in the context of the RPCA problem; in particular 
we describe previous work and algorithms for SPCP and RPCA in Section~\ref{P1C3:sec:literature},
discuss computationally efficient projections as optimization workhorses in Section~\ref{sec:projections},
and develop new accelerated projected quasi-Newton methods for the flipped 
and Lagrangian formulations in Section~\ref{sec:QN}. 
We then present a view of dual smoothing, describe the TFOCS algorithm, and show to to apply it to RPCA 
in Section~\ref{section-TFOCS}. We describe the general class of problems 
solvable by TFOCS in Section~\ref{sec:GenTFOCS}, detail the dual smoothing approach in Section~\ref{sec:dualSmoothing}, 
and present new convergence results in Section~\ref{sec:convergence}.
Finally, we demonstrate the efficacy of the new solvers and the overall formulation
on synthetic problems  and real data problems in Section~\ref{section-Experiments}.

\section{Convex Analysis Background} \label{section-analysis}
We work in finite dimensional spaces $\R^n$ (with Euclidean inner product) unless otherwise specified;
we note however that much of the general theory below generalizes immediately to Hilbert spaces and some of it to Banach spaces. 
Standard definitions are not referenced, but can be found in convex analysis textbooks~\citep{RTR,RTRW} or in review papers such as \citet{CombettesPesquetChapter}.

\subsection{Key definitions}
In this section, we provide definitions of objects that we use throughout the chapter. 

We work with functions that take on values from the extended real line $\Rbar := \R \cup \{ \infty \}$. 
For example, we define the indicator function as follows:
\begin{definition}[Indicator Function of a set $C$\index{Indicator function}]
\[
\chi_{C}(x) =  \begin{cases} 0 & x \in C \\ +\infty & x\notin C \end{cases} \]
\end{definition}
and thus for any functional $f$ on $\R^n$,
\[
\min_{x \in C}\; f(x) \;=\; \min_{x\in\R^n}\; f(x) + \chi_{C}(x).
\]
This allows a unified treatment of constraints and objectives by encoding constraints using indicator functions. 

The class $\Gamma_0(\R^n)$ denotes \textbf{convex}\index{convex}, \textbf{lower semi-continuous}\index{lower semi-continuous} (lsc), \textbf{proper}\index{proper} functionals from $\R^n$ to $\Rbar$. A function is lsc if and only if its graph is closed, and in particular a continuous function is lsc. 
A proper function is not identically equal to $+\infty$ and is never $-\infty$. We write $\dom f = \{x \mid f(x) <\infty\}$.
Further background is widely available, e.g., \citep{RTR,CombettesPesquetChapter}.

\begin{definition}[Subdifferential and subgradient\index{subdifferential}\index{subgradient}]
    Let $f\in\Gamma_0(\R^n)$, then the \textbf{subdifferential} of $f$ at the point $x\in \dom f$ is the set 
    \[ \partial f(x) = \{ d \in \R^n \mid \forall y\in\R^n,\; f(y) \ge f(x) + \iprod{d}{y-x} \} \]
and elements of the set are known as \textbf{subgradients}. 
\end{definition}
The sub-differential of an indicator function $\chi_C$ at $x$ is the normal cone to $C$ at $x$.
Fermat's rule is that $x\in\argmin \,f(x)$ iff $0\in\partial f(x)$, which follows by the definition of the subdifferential. If $f,g\in\Gamma_0(\R^n)$, then $\partial(f+g) = \partial f + \partial g$ in many situations (i.e., under constraint qualifications such as $f$ or $g$ having full domain).
In finite dimensions, G\^ateaux and Fr\'echet differentiability coincide on $\Gamma_0(\R^n)$ and $\partial f(x) = \{ d \}$ iff $f$ is differentiable at $x$ with $\nabla f(x)=d$.

We now introduce a key generalization of projections that will be used widely. 
\begin{definition}[Proximity operator\index{proximity operator}]
If $f\in \Gamma_0(\R^n)$ and $\lambda>0$, define
\[
\prox_{\lambda f}(y) = \argmin_x\; \lambda f(x)  +\frac{1}{2}\|x-y\|^2
 = (I+\lambda \partial f)^{-1}(y)
\]
\end{definition}    
Note that even though $\partial f$ is potentially multi-valued, the proximity operator is always uniquely defined (if $f\in\Gamma_0(\R^n)$), since it is the minimizer of a strongly convex function.  When we say that a proximity operator for $f$ is easy to compute, we mean that the proximity operator for $\lambda f$ is easy to compute for all $\lambda >0$. 
Computational complexity will be explored in more detail in subsequent sections.

The proximity operator generalizes projection, since $\prox_{\chi_{C}}(y) = \proj_C(y)$ where $\proj$ denotes orthogonal projection onto a set. Another example is the proximity operator of the $\ell^1$ norm, which is equivalent to soft-thresholding. The proximity operator is firmly non-expansive~\citep{CombettesPesquetChapter}, just like orthogonal projections. 

\begin{definition}[(Fenchel-Legendre) Conjugate function\index{conjugate function}]
    Let $f\in \Gamma_0(\R^n)$, then the conjugate function $f^*$ is defined
    \[ f^*(y) = \sup_{x}\; \iprod{x}{y} - f(x). \]
Furthermore, $\partial f^* = (\partial f)^{-1}$, where $(\cdot)^{-1}$ is the pre-image operation, and $f^{**} = f$.
\end{definition} 

\begin{definition}[Gauge]\index{gauge}
For a convex set $C$ containing the origin, 
the gauge $\gauge{x}{C}$ is defined by
\begin{equation}
\label{eq:gauge}
\gauge{x}{C} = \inf_\lambda \{\lambda: x \in \lambda C\}.
\end{equation}
\end{definition}

For any norm $\|\cdot\|$, the set defining it as a gauge is simply the unit ball $\mathbb{B}_{\|\cdot\|} = \{x: \|x\|\leq 1\}$.
Gauges are useful in our computational context since they easily allow some extensions, such as inclusion of non-negativity constraints. 

We make extensive use of the theory of dual functions. For example, 
if one can compute $\prox_f$, then one can compute $\prox_{f^*}$ and related quantities as well, using
\begin{align}\label{eq:dualProx}
\prox_{f^*}(x) = x - \prox_f(x) \\
\prox_{\check{f}}(x) = -\prox_f(-x) \label{eq:dualNegative}
\end{align}
where $\check{f}(x)=f(-x)$.

\begin{definition}[Relative interior]\index{relative interior}
The \textbf{relative interior} (ri)\index{relative interior} of a set $C\subset \R^n$ is the interior of $C$ relative to its affine hull\index{affine hull} (the smallest affine space containing $C$).
\end{definition}

\begin{definition}[Lipschitz continuity]\index{Lipschitz continuity}
A function $F:\R^n \rightarrow \R^m$ is Lipschitz continuous with constant $\Lipschitz$ if $\Lipschitz$ is the smallest real number such that for all $x,x' \in \R^n$,
\[
\| F(x) - F(x') \|_{\R^m} \le \Lipschitz\|x-x'\|_{\R^n}.
\]
\end{definition}

\section{Level-set methods for Residual-Constrained SPCP}
In this section, we discuss a range of convex formulations for SPCP,
their relationships, and survey prior art. We then show
how to apply level set methods to several of the formulations. 

\label{sec:variational}

\subsection{A primer on SPCP}

We  illustrate \eqref{eq:sum-SPCP} and \eqref{eq:max-SPCP} via different convex formulations. Flipping the objective and the constraints in~\eqref{eq:max-SPCP} and~\eqref{eq:sum-SPCP}, 
we obtain the following convex programs
\begin{equation} \label{eq:sum-SPCP-flip}  \tag{$\text{flip-SPCP}_\text{sum}$}
\begin{aligned}
& \min_{L,S}\; \rho(L+S-\obs)  \\
& \text{s.t.}\quad \matrixNorm{L}_* + \lambdaSum\|S\|_1 \le \tauSum
\end{aligned}
\end{equation}
\begin{equation} \label{eq:max-SPCP-flip}  \tag{$\text{flip-SPCP}_\text{max}$}
\begin{aligned}
& \min_{L,S}\; \rho(L+S-\obs)  \\
& \text{s.t.}\quad \max(\matrixNorm{L}_*, \lambdaMax\|S\|_1) \le \tauMax
\end{aligned}
\end{equation}
Solutions of~\eqref{eq:sum-SPCP-flip} and~\eqref{eq:max-SPCP-flip} 
are implicitly related to the solutions of~\eqref{eq:sum-SPCP} and~\eqref{eq:max-SPCP} via the Pareto frontier
by~\citet[Theorem 2.1]{AravkinBurkeFriedlander:2013}. While in many applications, 
$\rho(\cdot)$ is taken to be the 2-norm squared, the relationship holds in general. 
For the range of parameters where the constraints in~\eqref{eq:sum-SPCP} and~\eqref{eq:max-SPCP} are {\it active}, 
for any parameter $\eps$ there exist corresponding parameters $\tauSum(\eps)$ and $\tauMax(\eps)$, 
for which the optimal value of~\eqref{eq:sum-SPCP-flip} and~\eqref{eq:max-SPCP-flip} is $\eps$, 
and the corresponding optimal solutions $(\overline S_s, \overline L_s)$ and $(\overline S_m, \overline L_m)$ are also optimal for~\eqref{eq:sum-SPCP} 
and~\eqref{eq:max-SPCP}. 

For completeness, we also include the Lagrangian formulation:
\begin{equation} \label{eq:lag-SPCP} \tag{lag-SPCP}
\min_{L,S}\; \lambdaL \matrixNorm{L}_* + \lambdaS \|S\|_1 + \frac{1}{2}\|L+S-\obs\|_F^2 
\end{equation}

Problems
\eqref{eq:max-SPCP-flip} and \eqref{eq:sum-SPCP-flip} 
can be solved using projected gradient optimal projected gradient methods. 
The disadvantage of some of these formulations is that it is again not as clear how to tune the parameters. 
We show that one can solve~\eqref{eq:max-SPCP} and~\eqref{eq:sum-SPCP} 
using a sequence of flipped problems;
moreover this approach inherits the computational complexity guarantees 
of \eqref{eq:max-SPCP-flip} and \eqref{eq:sum-SPCP-flip}~\citep{aravkin2016level}.
In practice, better tuning also leads to faster algorithms, 
e.g., fixing $\eps$ ahead of time to an estimated 
`noise floor' greatly reduces the amount of required computation 
if parameters are to be selected via cross-validation. 

Finally, in some cases, it is useful to change the $\rho(L+S-\obs)$ term to $\rho(\AAA(L+S-\obs))$ 
where $\AAA$ is a linear operator. For example, let $\Omega$ be a subset of the indices of a $m \times n$ matrix. We may only observe $\obs$ restricted to these entries, denoted $\Proj_\Omega(\obs)$, in which case we choose $\AAA=\Proj_\Omega$. Most existing RPCA/SPCP algorithms adapt to the case $\AAA=\Proj_\Omega$ but this is due to the strong properties of the projection operator $\Proj_\Omega$. The advantage of our approach is that it seamlessly handles arbitrary 
linear operators $\AAA$. 

\subsection{Prior Art}
\label{P1C3:sec:literature}
While problem \eqref{eq:sum-SPCP} with $\eps=0$ has several solvers (e.g., it can be solved by applying the widely known 
Alternating Directions Method of Multipliers (ADMM)/ Douglas-Rachford method \citep{Combettes2007}), 
the formulation assumes the data are noise free. Unfortunately, the presence of noise we consider in this paper introduces a third term in the ADMM framework, 
where extra care must be taken to develop a convergent variant of ADMM~\citep{chen2013direct}. 
Interestingly,  there are only a handful of methods that can handle this case. 
Those using smoothing techniques no longer promote exactly sparse and/or exactly low-rank solutions. Those using dual decomposition techniques 
may require high iteration counts. 
Because each step requires a partial singular value decomposition (SVD) of a large matrix, 
it is critical that the methods only take a few iterations.

As a rough comparison, we start with related solvers that solve \eqref{eq:sum-SPCP} for  $\eps=0$.
\citet{RPCA_algo_Wright} solves an instance of \eqref{eq:sum-SPCP} with $\eps=0$ and a $800 \times 800$ system in $8$ hours. By switching to the \eqref{eq:lag-SPCP} formulation, \citet{GaneshRPCA} uses the accelerated proximal gradient method~\citep{BecTeb09} to solve a $1000 \times 1000$ matrix in under one hour. This is improved further in \citet{lin2010augmented} which again solves \eqref{eq:sum-SPCP} with $\eps=0$ using the augmented Lagrangian and ADMM methods and solves a $1500\times 1500$ system in about a minute. As a prelude to our results, our method can solve some systems of this size in about $10$ seconds (c.f.,~Fig.~\ref{P1C3fig:1}).

In the case of \eqref{eq:sum-SPCP} with $\eps > 0$, \citet{ASALM} propose the alternating splitting augmented Lagrangian method (ASALM), which exploits separability 
of the objective in the splitting scheme, and can solve a $1500 \times 1500$ system in about five minutes. 

The partially smooth proximal gradient (PSPG) approach of \citet{AybatRPCA} smooths just the nuclear norm term and then applies the well-known FISTA algorithm~\citep{BecTeb09}. 
\citet{AybatRPCA} show that the proximity step can be solved efficiently in closed-form, 
and the dominant cost at every iteration is that of the partial SVD.
They include some examples on video,  
solving $1500 \times 1500$ formulations in under half a minute.

The nonsmooth adaptive Lagrangian (NSA) algorithm of~\citet{Aybat2013} is a variant of the ADMM for \eqref{eq:sum-SPCP}, and makes use of the insight of~\citet{AybatRPCA}. 
The ADMM variant is interesting in that it splits the variable $L$, rather than the sum $L+S$ or residual $L+S-\obs$. 
Their experiments solve a 1500 $\times$ 1500 synthetic problems in between 16 and 50 seconds (depending on accuracy) .

\citet{Shen2014} develops a method exploiting low-rank matrix factorization scheme, maintaining $L = UV^T$.
This technique has also been effectively used in practice for matrix completion~\citep{JasonLee,Aravkin2013}, but lacks a full 
convergence theory in either context.  The method of~\citet{Shen2014} was an order of magnitude faster than ASALM, 
but encountered difficulties in some experiments where the sparse component dominated the low rank component. 
We note that the factorization technique may potentially speed up some of the methods presented here, but we leave this 
to future work, and only work with convex formulations.

\subsection{Level set methods for Sum-SPCP and Max-SPCP}

Recall that both~\eqref{eq:sum-SPCP} and~\eqref{eq:max-SPCP} can be written as follows:
\begin{equation}
\label{eq:general}
\min \phi(L, S) \quad \text{s.t.} \quad \rho\left(L + S - \obs\right) \leq \eps. 
\end{equation}
Earlier, we discussed that both $\phi$ and $\rho$ can be chosen by the modeler; 
in particular sum and max formulations come from choosing $\phiSum$ vs. $\phiMax$.
While classic formulations assume $\rho$ to be the Frobenius norm,  
this restriction is not necessary, and we consider $\rho$ to be smooth and convex. 
In particular, $\rho$ can be taken to be the robust Huber penalty~\citep{Hub}. 
Even more importantly, this formulation allows pre-composition of a smooth convex penalty 
with an arbitrary linear operator $\mathcal L$.  
In particular, note that the RPCA model is described by a simple linear operator: 
\begin{equation}
\label{eq:ids}
L + S = \begin{bmatrix} I & I \end{bmatrix} \begin{bmatrix} L \\ S \end{bmatrix}.
\end{equation}
Projection onto a set of observed indices $\Omega$ is also a simple linear operator that can be included in $\rho$. 
Operators may include different transforms (e.g., Fourier) applied to either $L$ or $S$. 

The problem class~\eqref{eq:general} falls into the class of problems studied by~\citet{BergFriedlander:2011,SPGL}
for $\rho(\cdot) = \|\cdot\|^2$ and by~\citet{AravkinBurkeFriedlander:2013} for arbitrary convex $\rho$. 
Following these references, we define  the {\it value function} $v(\tau)$ as
\begin{equation}
\label{eq:value}
v(\tau) = \min_{L, S} \rho\left(\mathcal L (L, S) - \obs\right) \quad
\text{s.t. } \phi(L,S) \leq \tau.
\end{equation}
This value function provides the bridge between formulations of type~\eqref{eq:general} and their `flipped'
counterparts. Specifically,  one can use Newton's method to find a solution to $v(\tau) = \eps$. 
The approach is agnostic to the linear operator $\mathcal L$ (it can be of the simple form~\eqref{eq:ids}, or include restriction 
in the missing data case, etc.).

For both formulations of interest, $\phi$ is a norm defined on a product space 
$\mathbb{R}^{n\times m} \times \mathbb{R}^{n\times m}$, since we can write 
\begin{eqnarray}
\phiSum(L, S) &= \left\|\begin{matrix} \matrixNorm{L}_* \\ \lambdaSum \|S\|_1 \end{matrix}\right\|_1, \\
\phiMax(L,S) &=  \left\|\begin{matrix} \matrixNorm{L}_* \\ \lambdaMax \|S\|_1 \end{matrix}\right\|_\infty.
\end{eqnarray}
In particular, both $\phiSum(L,S)$ and $\phiMax(L,S)$ are gauges as well as norms, 
and since we are able to treat this level of generality, we focus our theoretical results on this wider class.  

In order to implement Newton's method for~\eqref{eq:value}, the optimization problem to evaluate $v(\tau)$ 
must be solved (fully or approximately) to obtain $(\overline L, \overline S)$. 
Then the $\tau$ parameter for the next~\eqref{eq:value} problem is updated via 
\begin{equation}
\label{eq:newton}
\tau^{k+1} = \tau^k - \frac{v(\tau^k) - \epsilon}{v'(\tau^k)}. 
\end{equation}
Given $(\overline L, \overline S)$, $v'(\tau)$ can be written in closed form using
\citet[Theorem 5.2]{AravkinBurkeFriedlander:2013}, which simplifies to 
\begin{equation}
\label{eq:vprime}
v'(\tau) = -\phi^\circ(\mathcal{L}^T\nabla \rho(\mathcal{L} (\overline L, \overline S) - \obs)), 
\end{equation}
with $\phi^\circ$ denoting the polar gauge to $\phi$. 
The polar gauge is precisely $\gauge{x}{C^\circ}$, with
\begin{equation}
\label{eq:polar}
C^\circ = \{v: \left \langle v, x\right\rangle \leq 1 \quad \forall x \in C\}.
\end{equation}

In the simplest case, where $\mathcal L$ is given by~\eqref{eq:ids}, and $\rho$
is the least squares penalty, 
the formula~\eqref{eq:vprime} becomes 
\[
v'(\tau) = -\phi^\circ\left(\begin{bmatrix} \overline L + \overline S - \obs \\ \overline L + \overline S - \obs  \end{bmatrix}\right). 
\]

The main computational challenge in the approach outlined in~\eqref{eq:value}-\eqref{eq:vprime} 
is to design a fast solver to evaluate $v(\tau)$. Section~\ref{sec:QN} does just this. 

The key to RPCA is that the regularization functional $\phi$ 
is a gauge over the product space used to decompose $\obs$ into summands $L$ and $S$. 
This makes it straightforward to compute polar results for both $\phiSum$ and $\phiMax$.

\begin{theorem}[Max-Sum Duality for Gauges on Product Spaces]
    \label{thm:product-gauge}
    Let $\gamma_1$ and $\gamma_2$ be gauges on $\mathbb{R}^{n_1}$ and $\mathbb{R}^{n_2}$, and consider the function 
    \[
    g(x,y) = \max\{\gamma_1(x), \gamma_2(y)\}. 
    \]
    Then $g$ is a gauge, and its polar is given by 
    \[
    g^\circ(z_1,z_2) = \gamma_1^\circ(z_1) + \gamma_2^\circ(z_2). 
    \]
\end{theorem}

\begin{proof}
    Let $C_1$ and $C_2$ denote the canonical sets corresponding to gauges $\gamma_1$ and $\gamma_2$. It immediately follows 
    that $g$ is a gauge for the set $C = C_1 \times C_2$, since 
    \[
    \begin{aligned}
    \inf\{\lambda \geq 0| (x,  y) \in \lambda C\} &= \inf\{\lambda |  x \in \lambda C_1 \text{ and }y \in \lambda C_2\} \\
    &= \max \{\gamma_1(x), \gamma_2(y)\}.
    \end{aligned}
    \]
    By~\citet[Corollary 15.1.2]{Roc:70}, the polar of the gauge of $C$ is the support function of $C$, which is given by 
    \[
    \begin{aligned}
    \sup_{x \in C_1, y \in C_2} \left\langle(x, y), (z_1,z_2) \right \rangle &= \sup_{x \in C_1} \left\langle x, z_1 \right\rangle +  \sup_{y \in C_2} \left\langle y, z_2 \right\rangle \\
    &= \gamma_1^\circ(z_1) + \gamma_2^\circ(z_2).
    \end{aligned}
    \]
\end{proof}\vspace{-5mm}

This theorem allows us to easily compute the polars for $\phiSum$ and $\phiMax$ in terms of the polars of $\matrixNorm{\cdot}_*$ and $\|\cdot\|_1$,  
which are the dual norms of the spectral norm and infinity norm, respectively. 

\begin{corollary}[Explicit variational formulas for~\eqref{eq:sum-SPCP} and~\eqref{eq:max-SPCP}]
    \label{cor:explicit-polar}
    We have 
    \begin{equation}
    \label{eq:polarFormulas}
    \begin{aligned}
    \phiSum^\circ(Z_1, Z_2) &= \max\left\{\matrixNorm{Z_1}_2, \frac{1}{\lambdaSum} \|Z_2\|_\infty\right\}\\
    \phiMax^\circ(Z_1, Z_2) &= \matrixNorm{Z_1}_2 + \frac{1}{\lambdaMax} \|Z_2\|_\infty,
    \end{aligned}
    \end{equation}
    where $\matrixNorm{X}_2$ denotes the spectral norm (square root of the largest eigenvalue of $X^TX$)
    and $\|X\|_\infty$ denotes the largest entry in absolute value. 
    
\end{corollary}
We now have closed form solutions for $v'(\tau)$ in~\eqref{eq:vprime} for both formulations of interest.  
The remaining challenge is to design a fast solver for~\eqref{eq:value}
for formulations~\eqref{eq:sum-SPCP} and \eqref{eq:max-SPCP}.  We focus on this challenge in the remaining sections of the paper. 
We also discuss the advantage of~\eqref{eq:max-SPCP} from this computational perspective. 

\subsection{Projections}
\label{sec:projections}
In this section, we consider the computational issues of projecting onto the set defined by $\phi(L,S) \le \tau$. For $\phiMax(L, S) = \max(\matrixNorm{L}_*, \lambdaMax \|S\|_1)$ this is straightforward since the set is just the product set of the nuclear norm and $\ell_1$ norm balls, and efficient projectors onto these are known. In particular, projecting an $m \times n$ matrix (without loss of generality let $m \le n$) 
onto the nuclear norm ball takes $\order( m^2n )$ operations, and
projecting it onto the $\ell_1$ ball can be done on $\order( mn )$ operations using fast median-finding algorithms~\citep{knapsack1984,Duchi2008}.

For $\phiSum(L, S) = \matrixNorm{L}_* + \lambdaSum \|S\|_1$, the projection is no longer straightforward. Nonetheless, 
the following lemma shows this projection can be efficiently implemented. This lemma appears in \citet[Section 9]{BergFriedlander:2011}, but supply a proof here since it is not well-known. We begin with a basic proposition:
\begin{proposition}
 Projection onto the scaled $\ell_1$ ball, that is, $\{ x \in \R^d \mid \sum_{i=1}^d \alpha_i |x_i| \le 1 \}$ for some $\alpha_i > 0$, can be done in $\order( d \log(d) ) $ time.
\end{proposition}
We conjecture that fast median-finding ideas could reduce this to $\order(d)$ in theory, the same as the optimal complexity for the $\ell_1$ ball. 
The proof of the proposition follows by noting that the solution can be written in a form depending only on a single scalar parameter, and this scalar can be found by sorting $(|x_i|/\alpha_i)$ followed by appropriate summations.
Armed with the above proposition, we state an important lemma below. For our purposes, we may think of $S$ as a vector in $\R^{mn}$ rather than a matrix in $\R^{m\times n}$.

\begin{lemma} \label{lemma:jointProjection}
Let $L=U\Sigma V^T$ and $\Sigma = \diag(\sigma)$, and let $(S_i)_{i=1}^{mn}$ be any ordering of the elements of $S$. Then the projection of $(L,S)$ onto the $\phiSum$ ball is $(U\diag(\hat{\sigma})V^T,\hat{S})$, where $(\hat{\sigma},\hat{S})$ is the projection onto the scaled $\ell_1$ ball $\{ (\sigma,S) \mid\, \sum_{j=1}^{\min(m,n)} |\sigma_j| + \sum_{i=1}^{mn} \lambdaSum |S_i| \le 1 \}$.
\end{lemma}
\begin{proof}[Sketch of proof]
    We need to solve
    \begin{equation*}
    \min_{ \{ (L',S') \mid \, \phiSum(L',S') \le 1 \} }    
    \,\frac{1}{2}\norm{L'-L}_F^2 
    \end{equation*}
    Alternatively, solve
    \[
    \min_{S'}\, \left(\min_{\{L' \mid\, \matrixNorm{L'}_* \le 1 - \lambdaSum\|S'\|_1 \}}  \,\frac{1}{2}\norm{L'-L}_F^2 + \frac{1}{2}\norm{S'-S}_F^2\right).
    \]
The inner minimization is equivalent to projecting onto the nuclear norm ball, and this is well-known to be soft-thresholding of the singular values. Since it depends only on the singular values, recombining the two minimization terms gives exactly a joint projection onto a scaled $\ell_1$ ball.
\end{proof}

All the references to the $\ell_1$ ball can be replaced by the intersection of the $\ell_1$ ball and the non-negative cone, and the projection is still efficient.
As noted in Section~\ref{sec:variational}, imposing non-negativity constraints is covered by the gauge results of Theorem~\ref{thm:product-gauge} and Corollary~\ref{cor:explicit-polar}. 
Therefore,  the level set framework can be efficiently applied to this interesting case.


\subsection{Solving the max `flipped' sub-problem via projected quasi-Newton methods}
\label{sec:QN}

If we adopt $\phiMax$ as the regularizer, then the subproblem~\eqref{eq:max-SPCP-flip}
takes the explicit form 
\begin{equation} \begin{aligned}
& \min_{L,S}\; \frac{1}{2}\|L+S-\obs\|_F^2  \\
& \text{s.t.}\quad \matrixNorm{L}_* \leq \tauMax, \quad \|S\|_1 \le \tauMax/\lambdaSum.
\end{aligned}
\end{equation}
The computational bottle-neck is solving this problem quickly, once at each outer iteration. 
While first-order methods can be used, we can exploit the structure of the objective 
by using quasi-Newton methods. 
The main challenge here is that for the $\matrixNorm{L}_*$ term, it is tricky to deal with a weighted quadratic term 
(whereas for $\|S\|_1$, we can obtain a low-rank Hessian and solve it efficiently via coordinate descent).

Let $X=(L,S)$ be the full variable, so we can write the objective function as $f(X)=\frac{1}{2}\|\AAA(X)-\obs\|_F^2$. 
To simplify the exposition, we take  $\AAA=(I, I)$, 
but the presented approach applies to general linear operators (including terms like $\proj_\Omega$). 
The matrix structure of $L$ and $S$ is not yet important here, so we can think of them as 
reshaped vectors instead of matrices.

The gradient is $\nabla f(X) = \AAA^T( \AAA(X)-\obs )$. For convenience, we use $r(X) = \AAA(X)-\obs$ and
\[
\nabla f(X) = \pmat{ \nabla_L f(X) \\ \nabla_S f(X) } = 
\AAA^T\pmat{r(X)\\r(X)}, \quad r_k \equiv r(X_k). 
\]

The Hessian is  $\AAA^T\AAA = \pmat{I&I\\I&I}$.
We cannot simultaneously project $(L,S)$ onto their constraints with this Hessian scaling (doing so would solve the original problem!), since the Hessian removes separability.
Instead, we use $(L_k,S_k)$ to approximate the cross-terms.

The true function is a quadratic, so the following quadratic expansion around $X_k=(L_k,S_k)$ is exact:  
\[
\begin{aligned}
f(L,S) = f(X_k) &+ \iprod{\pmat{\nabla_L f(X_k)\\\nabla_S f(X_k)}}{\pmat{L-L_k\\S-S_k}} \\
&+ \frac{1}{2}\iprod{\pmat{L-L_k\\S-S_k}}{\nabla^2 f\pmat{L-L_k\\S-S_k}} \\
= f(X_k) &+ \iprod{\pmat{r_k\\r_k}}{\pmat{L-L_k\\S-S_k}} \\
&+ \frac{1}{2}\iprod{\pmat{L-L_k\\S-S_k}}{\pmat{I&I\\I&I}\pmat{L-L_k\\S-S_k}} \\
= f(X_k)& + \iprod{\pmat{r_k\\r_k}}{\pmat{L-L_k\\S-S_k}} \\
&+ \frac{1}{2}\iprod{\pmat{\mathbf{L}-L_k\\\mathbf{S}-S_k}}{\pmat{L-L_k+\mathbf{S}-S_k\\\mathbf{L}-L_k+S-S_k}}
\end{aligned}
\]
The coupling of the second order terms, shown in bold, prevents direct 1-step minimization of $f$, subject to the nuclear and 1-norm constraints. 
The FISTA method~\citep{BecTeb09} 
replaces 
the Hessian $\pmat{I&I\\I&I}$ with the upper bound $2\pmat{I&0\\0&I}$, 
which solves the coupling issue, but potentially loses too much second order information. 
For \eqref{eq:sum-SPCP-flip}, FISTA is about the best we can do (we actually use SPG~\citep{Wright2009} which did slightly better in our tests).
However, for \eqref{eq:max-SPCP-flip}---and for \eqref{eq:lag-SPCP},
which has no constraints but rather non-smooth terms, which can be treated like constraints
using proximity operators---the constraints are uncoupled
and we can take a ``middle road'' approach, replacing
\[
\iprod{\pmat{\mathbf{L}-L_k\\\mathbf{S}-S_k}}{\pmat{L-L_k+\mathbf{S}-S_k\\\mathbf{L}-L_k+S-S_k}}
\]
with
\[
\iprod{\pmat{L-L_k\\S-S_k}}{\pmat{L-L_k+\mathbf{S_k-S_{k-1}}\\\mathbf{L_{k+1}-L_{k}}+S-S_k}}.
\]
The first term is decoupled, allowing us to update $L_k$, and then this is plugged into the second term in a Gauss-Seidel fashion. In practice, we also scale this second-order term with a number slightly greater than $1$ but less than $2$ (e.g.,~$1.25$) which leads to more robust behavior. We expect this ``quasi-Newton'' trick to do well when $S_{k+1}-S_k$ is similar to $S_k - S_{k-1}$.

\section{Dual Smoothing and the Proximal Point method} \label{section-TFOCS}
\newcommand{\ff}{f} 
\newcommand{\FF}{F} 

This section describes the approach of the TFOCS algorithm~\citep{TFOCS} and its implementation\footnote{\url{http://cvxr.com/tfocs}}. 
The method is based on the proximal point algorithm and handles generic convex minimization problems. 
The original analysis in \citet{TFOCS} was in terms of convex \emph{cones}, 
but we re-analyze the method here in terms of extended valued convex \emph{functions}, and find stronger results. 

Secondly, we compare to alternatives in the literature; the basic ingredients involved in TFOCS are well-known in the optimization community, 
and there are many variants and applications. We discuss in detail the relationship with the family 
of preconditioned ADMM methods popularized by \citet{ChambollePock10}.  
The TFOCS algorithm also motivated the work of~\citealt*{DoubleSmoothing}, 
which promotes an alternative approach that smooths both the primal and the dual. 
This approach appeared to obtain stronger guarantees, but there is a price to pay,  
since in the context of sparse optimization, smoothing the primal necessarily yields a less sparse solution. 
We show that with the improved analysis of TFOCS, TFOCS in fact enjoys the same strong guarantees, 
while avoiding smoothing the primal. 

Finally, we apply this method to RPCA. The TFOCS formulation is flexible 
and can solve all standard variants of RPCA and SPCP, as well as incorporate  
non-negativity or other types of additional constraints. 
We briefly detail how the algorithm can be specialized for the RPCA problem.
Even without specializing the algorithm for RPCA, TFOCS has performed well. 
The results of tests from \citet{Bouwmans201422} are that ``LSADM~\citep{ShiqianFastADM} and TFOCS~\citep{TFOCS} 
solvers seem to be the most adapted ones in the field of video surveillance.''

\subsection{General form of our optimization problem}
\label{sec:GenTFOCS}
In this section, we provide a general notation that captures all optimization problems of interest. 
Note that even though we do not explicitly write constraints in this formulation, 
through the use of extended-value functions we capture constraints, and so 
in particular can express residual-constrained formulations using this notation. 

We consider the following generic problem
\begin{equation} \label{eq:generic1}
\min_x\; \smoothFcn(x) + \proxFcn_0(x) + \sum_{i=1}^m \proxFcn_i( \linear_i x - \offset_i )
\end{equation}
where 
\begin{itemize}
    \item $\smoothFcn$ and $\proxFcn_i$ for $i=0,\ldots,m$ are proper convex lsc functions on their respective spaces,
    \item $\smoothFcn$ is differentiable everywhere, with Lipschitz continuous gradient; note that we can consider $\smoothFcn( \linear x - \offset)$ trivially, since this is also differentiable,
    \item $\proxFcn_i$ for $i=0,\ldots,m$ has an easily computable proximity function,
    \item $\linear_i$ for $i=1,\ldots,m$ is a linear operator, and $b_i$ is a constant offset.
\end{itemize}
We distinguish $\proxFcn_0$ from $\proxFcn_i$, $i\ge 1$, since $\proxFcn_0$ is not composed with a linear operator. This is significant since being able to easily compute the proximity operator of $\proxFcn$ does not imply one can easily compute the proximity operator of $\proxFcn_i \circ \linear_i$ nor of $\proxFcn_i + \proxFcn_j$, so we deal with the $i=1,\ldots, m$ terms specially.

\begin{remark} \label{remark_offset}
Unlike the $\linear_i$ terms, the offsets $\offset_i$ can be absorbed into $\proxFcn_i$, since if $\widetilde{\proxFcn}(x) = \proxFcn(x-\offset)$   
then $\prox_{\widetilde{\proxFcn}}(x) = \offset + \prox_{\proxFcn}(x-\offset)$~\citep{CombettesPesquetChapter}. Thus we make these offsets explicit or implicit as convenient.
\end{remark}

RPCA in the general setting \eqref{eq:generalRPCA} can be recovered from the above by setting $x=(L,S)$, $\proxFcn_0(x) = \phi(L,S)$, $m=1$ and $\proxFcn_1 = \rho$ with $\linear_1 x = -L-S$ and $b_1=-\obs$, and $\smoothFcn=0$. In fact, many convex problems from science and engineering fit into this framework~\citep{CombettesPesquetChapter}. The strength of this particular model is that it is often easy to decompose a complicated function $f$ by a finite sum of  simple functions $\proxFcn_i$ composed with linear operators. In this case, $f$ many not be differentiable, and $\prox_f$ need not be easy to compute, so the model allows us to exploit the structure of the smaller building blocks.

\subsection{Dual smoothing approach}
\label{sec:dualSmoothing}
We re-derive the algorithm described in \citet{TFOCS}, but from a \emph{conjugate-function} viewpoint, whereas \citet{TFOCS} used a 
\emph{dual-conic} viewpoint. The later viewpoint is subsumed in the former, and is arguably less elegant.

Consider the problem 
\begin{equation} \label{eq:generic2}
\min_x\;  f(x) := \proxFcn_0(x) + \sum_{i=1}^m \proxFcn_i( \linear_i x - \offset_i )
\end{equation}
which is similar to \eqref{eq:generic1} but without the differentiable term $\smoothFcn$. In addition to our previous assumptions on these functions (convex, lsc, proper), we now assume that at least one minimizer exists (guaranteed if any function is coercive, or any set is bounded).

Our main observation is that instead of solving \eqref{eq:generic2} directly, 
we can instead use the proximal point method to minimize $f$, 
which exploits the fact that
\begin{equation}
\min_x \; \ff(x) = \min_{y} \; \left(\min_x \; \ff(x) + \frac{\mu}{2}\|x-y\|^2\right).
\end{equation}
This fact follows since $\ff(x) \le \ff(x) + \frac{\mu}{2}\|x-y\|^2$, and equality
is achieved by setting $y$ to one  of the minimizers of $\ff$.

Thus we solve a sequence of problems of the form
\begin{equation} \label{eq:smoothed}
\min_x \; \ff(x) + \frac{\mu}{2}\|x-y\|^2.
\end{equation}
for a fixed $y$. The exact proximal point method is
\begin{equation}\label{eq:proximalPoint}
y_{k+1} = \argmin_x\; \FF_{k}(x) := f(x) + \frac{\mu_k}{2}\|x-y_k\|^2
\end{equation}
where $\mu_k$ is any sequence such that $\limsup \mu_k <\infty$, and $y_0$ is arbitrary. Note that  $\FF_k$ depends on $\mu_k$ and $y_k$.

The benefit of \eqref{eq:smoothed} over \eqref{eq:generic2} is that $\FF_k$ is strongly convex, whereas $\ff$ need not be, and therefore the dual problem of $\FF_k$ is easy to solve, which we will make precise. 

Rewriting the objective, and ignoring offset terms $\offset_i$ for simplicity (see Remark~\ref{remark_offset}),
we have 
\begin{equation}\label{eq:smoothed2}
\min_x \; \underbrace{\proxFcn_0(x) + \frac{\mu}{2}\|x-y\|^2}_{\bigProxFcn(x)} + \sum_{i=1}^m \proxFcn_i(\linear_i x ).
\end{equation}
For $i=1,\ldots,m$, each $\linear_i$ is a linear operator from $\R^n$ to $\R^{m_i}$.
We can further simplify notation by defining a linear operator $\bigLinear$
and a vector $\bigX \in \R^{\sum_{i=1}^m m_i}$ (e.g., $\bigX = \bigLinear(X)$)
such that 
\[
\bigLinear(x) = \begin{pmatrix} L_1(x) \\ L_2(x) \\ \vdots \\ L_m(x) \end{pmatrix},
\quad
\bigX = \begin{pmatrix} z_1 \\ z_2 \\ \vdots \\ z_m \end{pmatrix}
\]
Then define 
\begin{equation}\label{eq:bigProx}
 \bigDualFcn(\bigX) = \sum_{i=1}^m \proxFcn_i(z_i),\;\text{so}\;
 \prox_{\bigDualFcn}(\bigX) = \begin{pmatrix} \prox_{\proxFcn_1}(z_1) \\ \prox_{\proxFcn_2}(z_2) \\ \vdots \\ \prox_{\proxFcn_m}(z_m) \end{pmatrix}
 \;\text{and}\;
 \prox_{\bigDualFcn^*}(\bigX) = \begin{pmatrix} \prox_{\proxFcn_1^*}(z_1) \\ \prox_{\proxFcn_2^*}(z_2) \\ \vdots \\ \prox_{\proxFcn_m^*}(z_m) \end{pmatrix}.
\end{equation}

We now rewrite \eqref{eq:smoothed2} in the following compact representation
\begin{equation} \label{eq:primal1}
\min_{x}\; \bigProxFcn(x) + \bigDualFcn(\bigLinear x).
\end{equation}
We are now in position to apply standard Fenchel-Rockafellar duality~\citep{RTR,Bauschke2011}
 to arrive at the dual problem
\begin{equation}\label{eq:dual} 
\min_{\dual} \; \underbrace{\bigProxFcn^*(\bigLinear^* \dual ) + \bigDualFcn^*(-\dual)}_{\dualFcn(\dual)}.
\end{equation}
Standard constraint qualifications for finite dimensional problems (e.g., Thm.\ 15.23 and Prop.\ 6.19x in \citet{Bauschke2011}) guarantee a zero duality gap if 
\[
\text{ri}\left(\dom \bigDualFcn\right) \cap \bigLinear \left( \text{ri}\left(\dom \bigProxFcn\right) \right) \neq \emptyset.
\]

The primal problem \eqref{eq:primal1} is not amenable to computation because even though we can calculate the proximity operators of $\bigDualFcn$ and $\bigProxFcn$, we cannot easily calculate the proximity operator of $\bigDualFcn \circ \bigLinear$. The dual formulation \eqref{eq:dual} circumvents this because instead of asking for the proximity operator of $\bigProxFcn^* \circ \bigLinear^*$, which is not easy, we will use its gradient, and in this case the linear term causes no issue. We can do this because $\bigProxFcn$ is at least $\mu$ strongly convex, so we have the following well-known result (see e.g., Prop.\ 12.60 in \citet{RTRW}).
\begin{lemma} 
The function $\bigProxFcn^*$ is continuously differentiable and the gradient is Lipschitz continuous with constant $\mu^{-1}$, and hence $\bigProxFcn^*\circ \bigLinear^*$ is also continuously differentiable with Lipschitz constant $\|\bigLinear\|^2/\mu$.
\end{lemma}
Note we are taking the operator norm of $\bigLinear$,
and $\matrixNorm{\bigLinear}^2 = \matrixNorm{\bigLinear \bigLinear^*}
= \sum_{i=1}^n \matrixNorm{\linear_i}^2$. The actual gradient is can be determined 
by exploiting the relation
\[
\nabla \bigProxFcn^*(w) = \partial \bigProxFcn^*(w) = \partial \bigProxFcn^{-1}(w),
\]
which follows from Fenchel's equality (see~\citet[Theorem 23.5]{RTR}),
i.e., if $x=\nabla \bigProxFcn^*(w)$ 
then $0 \in \partial \bigProxFcn(x) - w$, so $x$ minimizes $\bigProxFcn(\cdot) - \< w, \cdot \>$.
 Thus
\begin{align}
\nabla \bigProxFcn^*(w) &= \argmin_{x} \; \bigProxFcn(x) - \<x,w\> \notag \\
&= \argmin_x\; \proxFcn_0(x) + \frac{\mu}{2}\|x-y\|^2 - \<x,w\> \notag \\
&= \argmin_x\; \proxFcn_0(x) + \frac{\mu}{2}\|x-(y+w/\mu)\|^2  \notag \\
&= \prox_{\proxFcn_0/\mu}\left( y + w/\mu \right) \label{eq:gradient}
\end{align}
so we can calculate the gradient using $\prox_{\proxFcn_0}$. Furthermore, via the chain rule,
we have $\nabla (\bigProxFcn^* \circ \bigLinear^* )(\bigX)
= \bigLinear \nabla \bigProxFcn^*( \bigLinear^* \bigX )$.

Thus we have shown that the dual problem \eqref{eq:dual} is a sum of two functions, 
one of which has a Lipschitz continuous gradient and the other admits an easily computable proximity operator. 
Such problems can be readily solved via proximal gradient methods~\citep{ComWaj:05} and accelerated proximal gradient methods~\citep{BecTeb:09,BeckTeboulle14}.

If strong duality holds, then if $\bigX^\star$ is a solution of the dual problem, $x^\star = \nabla \bigDualFcn^*(\bigLinear^* \bigX^\star )$ is the unique solution to the primal problem (cf.~\citet[Prop.\ 19.3]{Bauschke2011}), and this was used in \citet{TFOCS} to motivate solving the dual. We actually have much stronger results that provide approximate optimality guarantees for approximate dual solutions. 

\begin{algorithm} 
\begin{algorithmic}[1]
\REQUIRE $\Lipschitz \ge \matrixNorm{\bigLinear}^2/\mu$ bound on Lipschitz constant; $\bigX_0$ arbitrary
\STATE $\bigY_1 = \bigX_0$, $t_1 = 1$.
\FOR{ $\;k = 1, 2, \ldots$}
\STATE Compute $\widetilde{x}_k=\nabla \bigProxFcn^*( \bigLinear^* \bigY_k)$ using \eqref{eq:gradient}
\STATE Set $\mathbf{G}=\bigLinear \widetilde{x}_k = \bigLinear \nabla \bigProxFcn^*( \bigLinear^* \bigY_k)$
\STATE $\bigX_k = -\prox_{\Lipschitz^{-1} \bigDualFcn^*}\left( -\bigY_k + \Lipschitz^{-1} \mathbf{G} \right)$ using \eqref{eq:bigProx} and \eqref{eq:dualProx}
\STATE $t_{k+1} = \frac{1+\sqrt{1+4t_k^2}}{2} \ge t_k+1/2$
\STATE $\bigY_{k+1} = \bigX_k + \frac{t_k-1}{t_{k+1}}\left( \bigX_{k} - \bigX_{k-1} \right)$
\ENDFOR
\end{algorithmic}
\caption{TFOCS method~\citep{TFOCS}, i.e.,~FISTA~\citep{BecTeb:09} applied to \eqref{eq:dual}; enforcing $t_k=1$ recovers proximal gradient descent}
\label{alg:fista}
\end{algorithm}

We can bound the rate of convergence of the dual objective function $q$:
\begin{theorem}[Thm.\ 4.4 in \citet{BecTeb:09}]
The sequence $(\bigX_k)$ generated by Algorithm~\ref{alg:fista} satisfies
\[
q(\bigX_k) - \min_{\bigX} q(\bigX) \le \frac{2 \Lipschitz d_0^2 }{k^2}
\]
where $d_0$ is the distance from $\bigX_0$ to the optimal set.
\end{theorem}

From this, we can recover a remarkable bound on the primal sequence.
\begin{theorem}[Thm.\ 4.1 in \citet{BeckTeboulle14}] \label{thm:inner}
Let $(x_k)$ be the sequence generated by
\[ x_k=\nabla \bigProxFcn^*( \bigLinear^* \bigX_k)
\] 
(similar to $\widetilde{x}_k$ but evaluated at $\bigX_k$ not $\bigY_k$).
Let $x^\star$ be the (unique) optimal point to \eqref{eq:primal1}.
Then
\[
\frac{\mu}{2}\|x_k - x^\star\|^2 \le q(\bigX_k) - \min_{\bigX} q(\bigX)
\]
for any point $\bigX_k$, and hence for $(\bigX_k)$ from Algorithm~\ref{alg:fista},
\begin{equation}\label{eq:rate1}
\|x_k - x^\star\| \le \frac{2 \sqrt{\Lipschitz} d_0 }{\sqrt{\mu} k}.
\end{equation}
\end{theorem}
Note that in practice, one typically uses $\widetilde{x}_k$ for any convergence tests, since it is a by-product of computation, whereas $x_k$ is expensive to compute. Since $0\le (t_k-1)/t_{k+1} < 1$ for $t_k \in [1,\infty)$, then if $\bigX_k$ converges, it follows that $\bigY_k$ converges to the same limit, so asymptotically $x_k \simeq \widetilde{x}_k$.

The result of the above theorem holds regardless of how the sequence $(\bigX_k)$ is generated, so if the dual method has better than worst-case convergence (or if an acceleration such as a  line search is used), then the primal sequence enjoys the same improvements.

Since $\sqrt{\Lipschitz} = \matrixNorm{\bigLinear}/\sqrt{\mu}$, combining with the other factor of $1/\sqrt{\mu}$ shows that $\|x_k - x^\star\| \propto 1/(\mu k)$, so choosing $\mu$ large leads to fast convergence (since the dual problem is very smooth). The trade-off is 
 that the outer loop (the proximal point method) will converge more slowly with $\mu$ large.

\subsection{Comparison with literature}
\label{sec-literature}

\paragraph{Dual methods}
The proposed method has been formulated in several contexts; part of the novelty of~\citet{TFOCS} is the generality of the method 
and pre-built function routines from which a wide variety of functions could be constructed. 
The basic concepts of duality and smoothing are widely used, and using duality to avoid difficult affine terms goes back to Uzawa's method (see \citep{Ciarlet89}) and the general concept of \emph{domain decomposition}\index{domain decomposition}.

More recent and specific approaches include those of \citep{combettesDualization,LiuToh09,PPPA,Necoura08} which particularly deal with signal processing problems. The work \citet{combettesDualization} considers a single smoothed problem, not the full proximal point sequence, and uses proximal gradient descent to solve the dual. They establish convergence of the primal variable but without a bound on the convergence rate. Explicit use of the proximal point algorithm is mentioned in \citet{LiuToh09}, which focuses on the nuclear norm minimization problem, but uses Newton-CG methods, which requires a third level of algorithm hierarchy and good heuristic values for stopping criteria of the conjugate-gradient method. 
The algorithm SDPNAL~\citep{SDPNAL} is similar to \citet{LiuToh09} and uses a Newton-CG augmented Lagrangian framework to solve general semi-definite programs (SDP). 
The work of \citet{PPPA} focuses on a more specific version of the problem but contains the same general ideas, and has inner and outer iterations. The algorithm in \citet{Necoura08} focuses on standard ADMM settings that have non-trivial linear terms, and smooths the dual problem; they follow \citet{Nesterov05} and have specific complexity bounds when the appropriate constraint set is compact.

\paragraph{Primal-dual methods}
Another method to remove effects of the linear terms $\linear_i$ is to solve a primal-dual formulation.
Many of these are based on duplicating variables into a product space and then enforcing \emph{consensus}\index{consensus} of the duplicated variables. This can replace many terms, such as $\sum_{i=1}^m \proxFcn_i(\linear_i x)$, with two generalized functions, which is inherently easier, since it is often amenable to the Douglas-Rachford algorithm~\citep{Combettes2007}.
Specific examples of this approach are  
\citep{CombettesPesquet12,combettes2013systems,infConv_paper,boct2013solving,botSIOPT,botMathProg,botFBF}.
The paper by \citet{MonotoneSkew} is slightly unique in that it reformulates the primal-dual problem into one that can be solved by the obscure forward-backward-forward algorithm of Tseng (the forward-backward algorithm does not apply since in the primal-dual setting, Lipschitz continuity does not imply co-coercivity).

Another main line of primal-dual methods was 
motivated in \citet{preconditionedADMM} 
 as a preconditioned variant of ADMM and then analyzed in \citet{ChambollePock10},
 and an improved analysis by \citet{HeYuan10} allowed a generic formulation
 to be proposed independently by \citet{Vu2011,Condat2011}.
 In more particular settings, it is known as \emph{linearized} ADMM\index{linearized ADMM} or primal-dual hybrid gradient\index{primal-dual hybrid gradient} ({PDHG}), and has seen a recent surge of interest and analysis 
\citep{NIPS2014_5494_Volkan,HeYuanADMMrate,LinADMM_LiEtAl,LinADMM_RenEtAl,LinADMM_WangYuan,LinADMM_YangYuan,LinADMM_ZhangBurgerOsher,YALL1,pdhg_2013}.
Several recent survey papers~\citep{SPmag,Duality2015} review these algorithms in more detail.

The PDHG has not been applied specifically to the RPCA problem to our knowledge. The next section describes this method in more detail.

\paragraph{Detailed comparison with PDHG}
On certain classes of problems, our approach is quite similar to the PDHG approach.
Consider the following simplified version of \eqref{eq:generic2}:
\[
\min_x \proxFcn_0(x) + \proxFcn_1( \linear x - \offset )
\]
At each step of the proximal point algorithm, we minimize the above objective perturbed by $\mu/2\|x-y\|^2$, where $\mu>0$ is arbitrary. Applying FISTA to the dual problem leads to steps of the following form:
\begin{align*}
x_{k+1} &= \argmin_x \; \proxFcn_0(x) - \iprod{ \overline{\smalldual}}{ \linear x - \offset } + \frac{\mu}{2}\|x-y\|^2 \\
\smalldual_{k+1} &= \argmin_{\smalldual} \; \proxFcn_1^*(\smalldual) - \iprod{ \smalldual}{\linear x_{k+1} } + \frac{1}{2t}\|\smalldual - \smalldual_k\|^2 \\
\overline{\smalldual} &= \smalldual_{k+1} + \theta_k\left( \smalldual_{k+1} - \smalldual_k \right)
\end{align*}
where $\theta_k = (t_k-1)/t_{k-1}$ as in Algorithm~\ref{alg:fista} (and $\lim_{k\rightarrow \infty} \theta_k=1$). We require the stepsize $t$ to satisfy $t \le \mu/\Lipschitz^2$.

For the PDHG method, pick stepsizes $\tau \sigma < 1/\Lipschitz^2$. There is no outer loop over $y$, and the full algorithm is:
\begin{align*}
\smalldual_{k+1} &= \argmin_{\smalldual} \; \proxFcn_1^*(\smalldual) - \iprod{ \smalldual}{\linear \overline{x} } + \frac{1}{2\sigma}\|\smalldual - \smalldual_k\|^2 \\
x_{k+1} &= \argmin_x \; \proxFcn_0(x) - \iprod{ \overline{\smalldual}}{ \linear x - \offset } + \frac{1}{2\tau}\|x-x_k\|^2 \\
\overline{x} &= x_{k+1} + \theta\left( x_{k+1} - x_k \right)
\end{align*}
with $\theta=1$ (see Algorithm~1 in \citet{ChambollePock10}).

The two algorithms are extremely similar, the main differences being that the TFOCS approach updates $y$ occasionally, while PDHG updates $y=x_k$ every iteration of the inner algorithm and thus avoids the outer iteration completely. The lack of the outer iteration is an advantage, mainly since it avoids the issue of a stopping criteria. However, the advantage of an inner iteration is that we can apply an accelerated Nesterov method, which can only be done in the PDHG if one has further assumptions on the objective function.

We present a numerical comparison of the two algorithms applied to  RPCA in Figures \ref{fig-PDHG1} and \ref{fig-PDHG2}.

\paragraph{Detailed comparison with double-smoothing approach}
For minimizing smooth but not strongly convex functions $f$, classical gradient descent generates a sequence of iterates $(x_k)$ such that the objective converges at rate $f(x_k) - \min_x f(x) \le \order(1/k)$, and $(x_k)$ itself converges, with $\|\nabla f(x_k)\| \le \order(1/\sqrt{k})$. The landmark work of \citet{Nesterov83} showed that a simple acceleration technique similar to the heavy-ball method generates a sequence $(x_k)$ such that $f(x_k) - \min_x f(x) \le \order(1/k^2)$, although there are no guarantees about convergence of the sequence\footnote{
There is no experimental evidence that the sequence does not converge, and indeed a recent preprint shows that a slight modification of the algorithm can guarantee convergence\citep{FISTA_converges}}   
$(x_k)$ nor strong bounds on $\|\nabla f(x_k)\|$. Note that our dual scheme uses FISTA~\citep{BecTeb09} which is a generalization of Nesterov's scheme; we refer to any such scheme with $\order(1/k^2)$ as ``accelerated''.

 Defining an $\epsilon$-solution to be a point $x$ such that $f(x) - \min_x f(x) \le \epsilon$, we see that it takes $\order(1/\epsilon)$ and $\order(1/\sqrt{\epsilon})$ iterations to reach such a point using the classical and accelerated gradient descent schemes, respectively.

In 2005 Nesterov introduced a smoothing technique~\citep{Nesterov05} which applies to minimizing non-smooth functions over compact sets. A naive approach would use sub-gradient descent, in which the worst-case convergence of the objective is $\order(1/\sqrt{k})$. By smoothing the primal function by a sufficiently small fixed amount, then using the compact constraints, the smooth function differs from the original function by less than $\epsilon/2$. One can then apply Nesterov's accelerated method which generates an $\epsilon/2$-optimal point (to the smoothed problem, and hence $\epsilon$ optimal to the original problem ) in $\order(\Lipschitz/\sqrt{\epsilon})$ iterations. The catch is that $\Lipschitz$ is the Lipschitz constant of the smoothed objective, which is proportional to $\epsilon^{1/2}$, so the overall convergence rate is $\order(1/\epsilon)$, which is still better than the subgradient schemes that would take $\order(1/\epsilon^2)$.

The aforementioned smoothing technique is an alternative to our approach, but the two are not directly comparable, since we do not assume the objective has the same form, nor is our domain necessarily bounded. Furthermore, smoothing the primal objective can have negative consequences. For example, it is common to solve $\ell_1$ regularized problems in order to generate sparse solutions. Our dual-smoothing technique keeps the primal non-smooth and therefore still promotes sparsity, whereas a primal-smoothing technique would replace $\ell_1$ with the Huber function, and this does not promote sparsity; see Fig. 1 in~\citet{TFOCS}.

Another option is the \emph{double-smoothing} technique proposed by~\citet{DoubleSmoothing}. 
This is the approach most similar with our own. Like our approach, a strongly convex term is added to the primal problem in order to make the dual problem smooth, and then the dual problem is solved with an accelerated method. Departing from our approach, they additionally smooth the primal as well (as in \citet{Nesterov05}), which makes the dual problem strongly convex. The reason for this is subtle. Without the strong convexity in the dual (i.e., our approach), we only have a bound on the dual objective function. To translate this into a bound on the primal variable, measured in terms of objective function or distance to the feasible set, requires using the gradient of the dual variable. As mentioned above, accelerated methods have faster rates of convergence in the objective but not of the gradients. For this reason, one must resort to a classical gradient descent method which has slower rates of convergence.

Making the dual problem strongly convex allows the use of special variants of Nesterov's accelerated method (see \citet{Nesterov2004}) which converge at a linear rate, and, importantly, so do the iterates and their gradients. The convergence is in terms of the smooth and perturbed problem, so the size of these perturbations is controlled in such a manner (again, the domains are assumed to be bounded) such that one recovers a $\order( 1/\epsilon \log(1/\epsilon) )$ convergence rate.

The analysis of \citet{DoubleSmoothing} suggests that our method of single smoothing is flawed, but this seems to be an artifact of previous analysis. Using Theorem~\ref{thm:inner}, which is a recent result, we have a rate on the convergence of the primal sequence, rather than on its objective value or distance to optimality. In many situations this is a stronger measure of convergence, depending on the purpose of solving the optimization problem. For robust PCA, the distance to the true solution is indeed a natural metric, whereas sub-optimality of the objective function is rather artificial, and distance to the feasible set depends on the choice of model parameters which maybe somewhat arbitrary.

Furthermore, the lack of bounds on the iterate \emph{sequence} generated by an accelerated method, which was the issue in the analysis of \citet{DoubleSmoothing}, is mainly a theoretical one, since in most practical situations, the variables and their gradients do appear to converge at a fast rate. Our situation is also different since the constraints need not be compact, and we use the proximal point method to reduce the effect of the smoothing.

\subsection{Effect of the smoothing term}
The next section discusses convergence of the proximal point method, but we first discuss the phenomenon that sometimes, the proximal point method converges to the exact solution in a finite number of iterations. This case is not covered by classical exact penalty results~\citep{BertsekasBook},  which only apply when the perturbation is non-smooth (e.g., $\|x-y\|$), whereas we use a smooth perturbation $\|x-y\|^2$.

Whenever the functions and constraints are polyhedral, such as for linear programs, finite convergence (or the ``exact penalty'' property) will occur. This was known since the 1970s; see Prop.\ 8 in~\citet{Rockafellar1976} and \citep{Mangasarian79,Poljak74,Bertsekas75}.
The special case of noiseless basis pursuit was recently analyzed in \citet{YinPenalty} using different techniques.  
More general results, allowing a range of penalty functions, were proved in \citet{FrieTsen:2007}.

For non-polyhedral problems, exact penalty does not occur in general. For example, one can construct an example of nuclear norm minimization which does not have the exact penalty property. However, under additional assumptions that are typical to guarantee exact recovery in the sense of \citet{candes2010matrix}, it is possible to obtain exact penalty results. Research in this is motivated by the popularity of the \citep{SVT} algorithm, which is a special case of the TFOCS framework applied to matrix completion. Results are in \citet{YinNuclearPenalty}, as well as \citet{ExactPenaltyRPCA} (and the correction \citet{ExactPenaltyRPCA_fix}) which also provides results for the RPCA problem in particular. Some results on generalizations to tensors are also available~\citet{ExactPenaltyTensor}.

\subsection{Convergence}
\label{sec:convergence}

\paragraph{Certificates of accuracy of the sub-problem}
For solving the smoothed sub-problem $\min_x \,\FF_k(x)$, we assume the proximity operator of each $\proxFcn_i$ is easy to compute, and given this, it is reasonable to expect that it is easy to compute a point in $\partial \proxFcn_i$ as well. Furthermore, since the algorithm  computes the effect of the linear operators  on the current iterate, 
it may be possible to reuse previously computed values. Thus, computing a point in $\partial \FF_k$ may be relatively cheap since it is just the sum of the $\partial \proxFcn_i$ composed with the appropriate linear operators.  We can now obtain accuracy guarantees via the following proposition.

\begin{proposition}\cite[Prop. 3]{Rockafellar1976}\label{prop:guarantee}
    Let $y = \argmin_x \, \FF_k(x)$, then for all points $x$ and all $d\in \partial F_k(x)$,
\begin{equation}
\|x-y\| \le \mu_k^{-1}\|d\|.
\end{equation}
 \end{proposition}

\paragraph{Convergence of the proximal point method}
The convergence of the proximal point method is well-understood, but we are particularly interested in the case when the update step is computed inaccurately. There has been recent work on this (see e.g. \citep{SchmidtInexact,VillaInexact,Devolder2011})
but often under the assumption that the computed point is feasible, i.e., 
it is inside $\dom f$. Using the dual method, this cannot be guaranteed in general (though it certainly applies to many special cases).
One can apply the analysis of gradient descent from \citep{Devolder2011} to the proximal point algorithm (viewed as gradient descent on the Moreau-Yosida envelope), and compute an inexact gradient in the sense that the primal point is the exact gradient of a perturbed point. This perturbed point is based on the sub-optimality of the dual variable (see \eqref{eq:gradient}), which, per the discussion of \citet{DoubleSmoothing} above, does not have a bounded converge rate when using an accelerated algorithm, 
and hence we do not pursue this line of analysis.

We start with an extremely broad theorem that guarantees convergence under minimal assumptions, albeit without an explicit convergence rate:
\begin{theorem}\cite[Thm. 1]{Rockafellar1976} \label{P1C3thm:1}
The approximate proximal point method defined by
\begin{align*}
\tilde{y}_{k+1} &= \argmin\, \FF_k(x) \\
y_{k+1} &\;\text{any point satisfying}\; \|y_{k+1} - \tilde{y}_{k+1}\| \le \epsilon_k \\
\end{align*}
with $\sum_{k=1}^\infty \epsilon_k < \infty$, $y_0$ arbitrary, $\FF_k$ as defined in \eqref{eq:proximalPoint}, and $\limsup_{k\rightarrow \infty} \mu_k <\infty$, 
will generate a sequence $\{y_k\}$ that converges to a minimizer of $f$.
\end{theorem}
We note that the boundedness of iterates follows by our assumption that a minimizer exists; in infinite dimensions, convergence is in the weak topology. 
To guarantee $\|y_{k+1} - \tilde{y}_{k+1}\| \le \epsilon_k $, we can either bound this \emph{a priori} using  Thm.~\ref{thm:inner}, or we can bound it \emph{a posteriori} by explicitly checking using Prop.~\ref{prop:guarantee}.

We state a second theorem that guarantees local linear convergence under standard assumptions. This assumption is that there is a unique solution to $\min\, \ff(x)$ and that $f$ has sufficient curvature nearby; it is related to the standard second-order sufficiency condition, but slightly weaker. See \citet{Rockafellar1976} for an early use, and \citet{BurkeQian97} for a more recent discussion.
\begin{assumption}\label{assumption1}
    There is a unique solution $x^\star$ to $\min\, f(x)$, i.e., $\partial \ff^{-1}(0) = x^\star$; and 
$\partial \ff^{-1}$ is locally Lipschitz continuous at $0$ with constant $a$, i.e., there is some $r$ such that $\|w\|\le r$ implies $\|x-x^\star\|\le a\|w\|$ whenever $x \in \partial \ff^{-1}(w)$.
\end{assumption}
Recall that via basic convex analysis, $0 \in \partial \ff(x^\star)$.
Finding $\argmin F_k(x)$ is the same as computing the proximity operator $P_k \defeq (I+ \mu_k^{-1} \partial \ff )^{-1}$. Define $Q_k = I - P_k$, then we have that $P_k$ (and $Q_k$) are firmly non-expansive~\citep{Moreau1965,Bauschke2011}, meaning
\[
\forall x,x',\; \|P_k(x)-P_k(x')\|^2+\|Q_k(x)-Q_k(x')\|^2 \le \|x-x'\|^2.
\]
Furthermore, $x^\star = P_k(x^\star)$ (this is independent of $\mu_k$), and $Q_k(x^\star) = 0$.  Now, we state a novel theorem, where for simplicity we have assumed $\mu_k \equiv \mu \neq 0$:

\begin{theorem} \label{P1C3thm:2}
    Under Assumption~\ref{assumption1}, 
    the approximate proximal point method defined by
    \begin{align*}
    \tilde{y}_{k+1} &= \argmin\, \FF_k(x) \\
    y_{k+1} &\;\text{any point satisfying}\; \|y_{k+1} - \tilde{y}_{k+1}\| \le (\gamma/2)^k \\
    \end{align*}
    with 
    \[
    \gamma = \frac{a}{\sqrt{a^2 + \mu^{-2}}} < 1
    \]
    generates a sequence $(y_k)$ that 
    converges linearly to $x^\star$ for all $k$ sufficiently large, and with rate $\gamma$.
\end{theorem}
\begin{proof}
    Note that we have defined $\tilde{y}_{k+1} = P_k(y_k)$. 
Observe that the assumption of Theorem~\ref{P1C3thm:1} holds since the errors are clearly summable, hence $(y_k)$ converges, and $\|y_{k+1} - y_k \| \rightarrow 0$, so this is arbitrarily small for $k$ sufficiently large. We also have
\begin{align*}
\|Q_k(y_k)\| = \|y_k - P_k(y_k)\| &\le \|y_k - y_{k+1}\| + \|y_{k+1} - P_k(y_k)\| 
\end{align*}
and both the terms on the right side go to zero.
By basic convex analysis (e.g., Prop.\ 1a in \citet{Rockafellar1976}),
\[
P_k(y_k) \in \partial \ff^{-1}( \mu Q_k(y_k) )
\]
and so for $k$ large enough (say, $k\ge k_0$), we are in the Lipschitz region of the assumption, so
\begin{equation}\label{eq:localLip}
\|P_k(y_k) - x^\star\| \le a \| \mu Q_k(y_k) \|.
\end{equation}
Now using the firmly non-expansiveness and properties of $x^\star$ mentioned above,
\begin{equation}\label{eq:firm}
\|x^\star -P_k(y_k)\|^2 + \|Q_k(y_k)\|^2
 \le \|x^\star - y_k\|^2
\end{equation}
so combining \eqref{eq:localLip} with \eqref{eq:firm} gives
\[
\|x^\star -P_k(y_k)\| \le \gamma \|x^\star - y_k\|
\]
which in effect proves the eventual linear convergence in the exact case where $y_{k+1} = P_k(y_k)$ 
(up to this point, the proof follows Thm.\ 2 from \citet{Rockafellar1976}).

Now bound
\begin{align*}
\|y_{k+1} - x^\star \| &\le \|y_{k+1} - P_k(y_k)\| + \|P_k(y_k) - x^\star\| \\
&\le (\gamma/2)^k + \gamma \|y_k - x^\star \| \\
&\le (\gamma/2)^k + \gamma \left( (\gamma/2)^{k-1} + \gamma \|y_{k-1} - x^\star \| \right) \\
& \vdots \\
&\le \gamma^k \sum_{i=k_0}^k 2^{-i} +  \gamma^{k+1-k_0}\|y_{k_0} - x^\star\| \\
&\le \gamma^k \left(1+\gamma^{1-k_0}\|y_{k_0} - x^\star\| \right)
\end{align*}
which proves our result.
\end{proof}

Again, we can certify that $y_{k+1}\simeq P_k(y_k)$ either use Thm.~\ref{thm:inner}, or we can bound it \emph{a posteriori} by explicitly checking using Prop.~\ref{prop:guarantee}. Since the linear convergence only occurs locally, it is not possible to provide an overall iteration-complexity of the inner and outer iterations (it is possible with further assumptions on $f$, such as $f$ having full domain; see \citet{BeckerThesis}). Without some form of strong convexity near the solution, it is generally not possible to bound the rate on the iterates, but rather only bound the rate of the objective function, and this is not possible with the dual approach since the point may not be feasible.

If we assume that the linear converges occurs globally, then we can combine this with our complexity bound on the sub-problem from \eqref{eq:rate1}. Converting that rate to our new notation, and using $j$ to index the inner loop, and setting the initial dual variable $\bigX_0$ to the one corresponding to $y_k$, we have
\[
\| x_{j} - P_k(y_k) \| \le 2\frac{ \|\bigLinear\| d_k }{\mu \cdot j }
\]
where $d_k$ is the distance from $\bigX_0$ (corresponding to $y_k = \nabla \bigProxFcn^*( \bigLinear^* \bigX_0)$) to the optimal set of dual solutions. Bounding this explicitly can be done in some cases (see \citet{NIPS_Bruer2014}), but for the sake of analysis we will simply assume $d_k$ is upper-bounded by some $d$.

\subsection{Solving the dual problem efficiently in the case of RPCA}
For solving \eqref{eq:classicRPCA}, we set $\proxFcn_0(L,S) = \matrixNorm{L}_* + \lambda \|S\|_1$
and $\proxFcn_1(\linear(L,S)-\obs) = \chi_{\{0\}}( L+S-\obs)$ to enforce the constraint exactly.
In this case, $\proxFcn_1^*$ is the constant function, and so $\prox_{\proxFcn_1^*}$ is the identity --- that is to say, the dual problem is unconstrained.

In that case, instead of using FISTA to solve the dual, one may use techniques from \emph{unconstrained} optimization, such as non-linear conjugate gradient and L-BFGS~\citep{NocedalWright}. These algorithms work extremely well in practice. We do not go into further detail since we find the exact constraint formulation of RPCA to be artificial. With inequality constraints, the dual problem becomes non-smooth so it is necessary to use a proximal gradient method. Due to the cost of the objective function, it may be worthwhile to use quasi-Newton projected gradient methods such as that of \citet{PQN}.

\section{Numerical Experiments} \label{section-Experiments}

\subsection{Numerical results for TFOCS}
To highlight the flexibility of TFOCS, we consider a background subtraction problem of a surveillance video in which we do not wish to enforce $\obs = L+S$ but instead we wish to separate $\obs$ into components up to the quantization level. The video $\obs$ is quantized to integer values between $0$ and $255$, so we can think of this as being the quantized version of some real-valued video $\widehat{\obs}$, and thus $\| \widehat{\obs}-\obs\|_\infty \le 0.5$ since the quantized version is rounded to the nearest integer. Hence we solve the following:
\begin{equation}\label{linf}
\min_{L,S}\; \phi_0(L,S) = \matrixNorm{L}_* + \lambda\|S\|_1 \quad\text{s.t.}\quad \|\obs-L-S\|_\infty \le 0.5.
\end{equation}

In the TFOCS software (written in MATLAB), we work with the primal variable \texttt{X=\{L,S\}}, and one specifies the fucntion $\proxFcn_0$ in the term \texttt{obj} like
\begin{verbatim}
obj    = { prox_nuclear(1), prox_l1(lambda) };    
\end{verbatim}
Next we encode $\proxFcn_1$ and $\linear_1$ and $\offset_1$.
The linear term, applied to \texttt{X=\{L,S\}}, is $\linear_1=[I,I]$, and the offset is $\offset_1 = -\obs$. This is encoded as
\begin{verbatim}
affine = { 1, 1, -A };
\end{verbatim}
and $\proxFcn_1$ is represented implicitly by giving its conjugate. For standard equality constrained RPCA, the constraint is $\obs -L-S=0$, 
so $\proxFcn_1$ is the indicator function of the set $\{0\}$, and the conjugate of this is the function that is constant everywhere, so its proximity operator is the identity (i.e., projection on $\R^n$). This is written
\begin{verbatim}
dualProx = proj_Rn
\end{verbatim}

If instead we wish to solve \eqref{linf}, then $\proxFcn_1$ is the indicator function of the $\ell_\infty$ ball of radius $1/2$, and so its conjugate is $1/2 \| \cdot\|_1$, which is written 
\begin{verbatim}
dualProx = prox_l1(0.5);
\end{verbatim}
The code can then be called as follows:
\begin{verbatim}
X = tfocs_SCD( obj, affine, dualProx, mu, X0 );
\end{verbatim}

\begin{figure}[ht]
\centering
\subfigure[$\obs$]{\includegraphics[width=.25\textwidth]{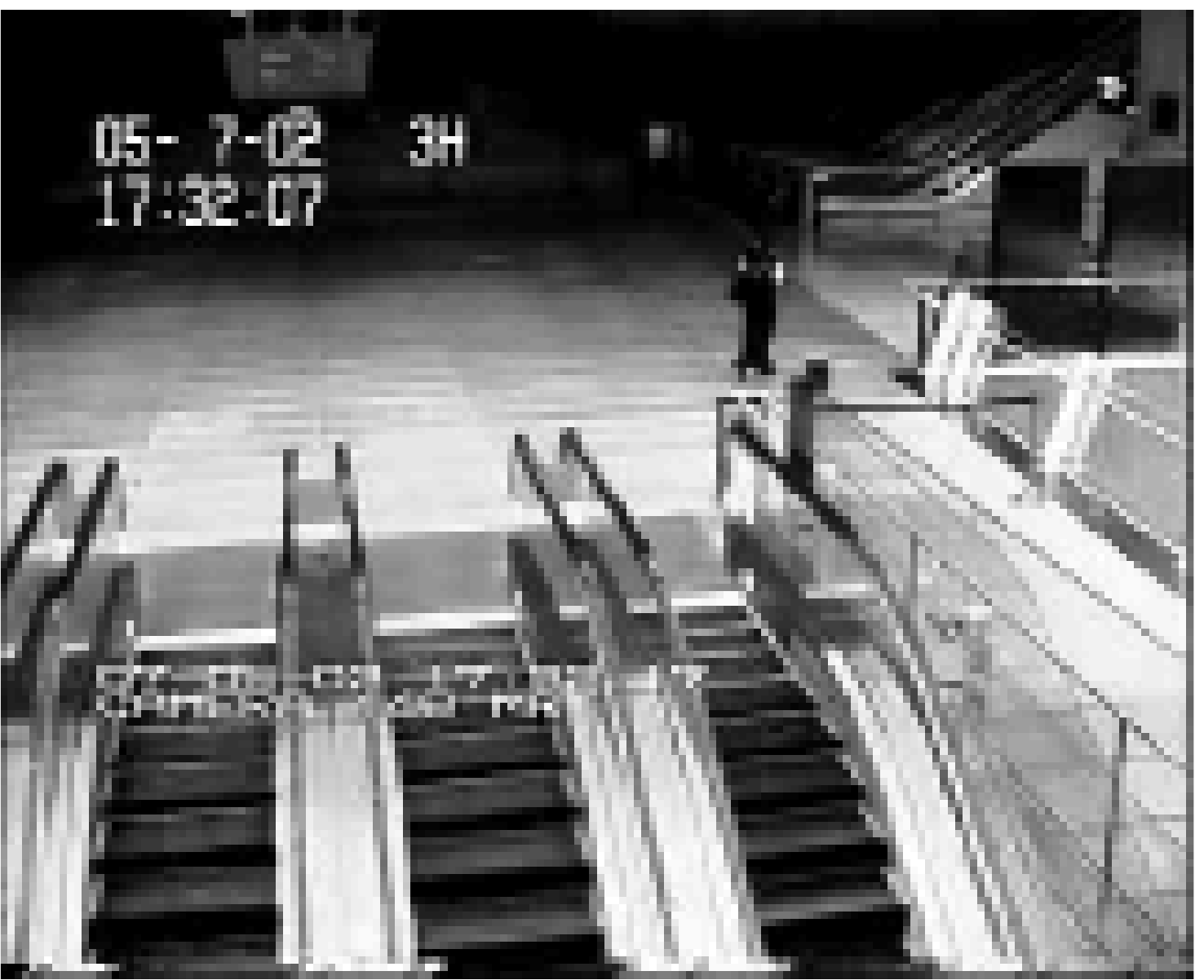}}
\subfigure[$L$]{\includegraphics[width=.25\textwidth]{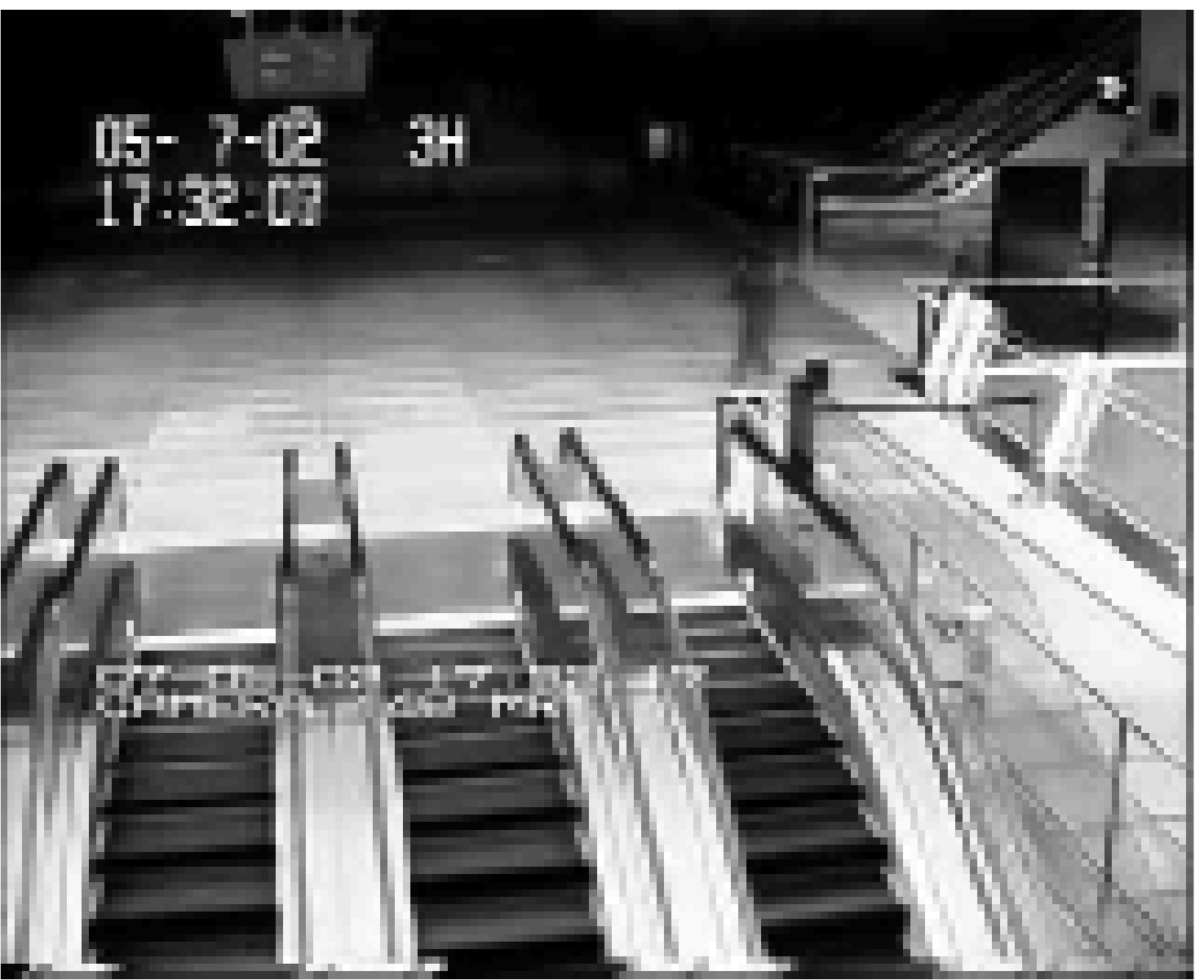}}
\subfigure[$S$]{\includegraphics[width=.25\textwidth]{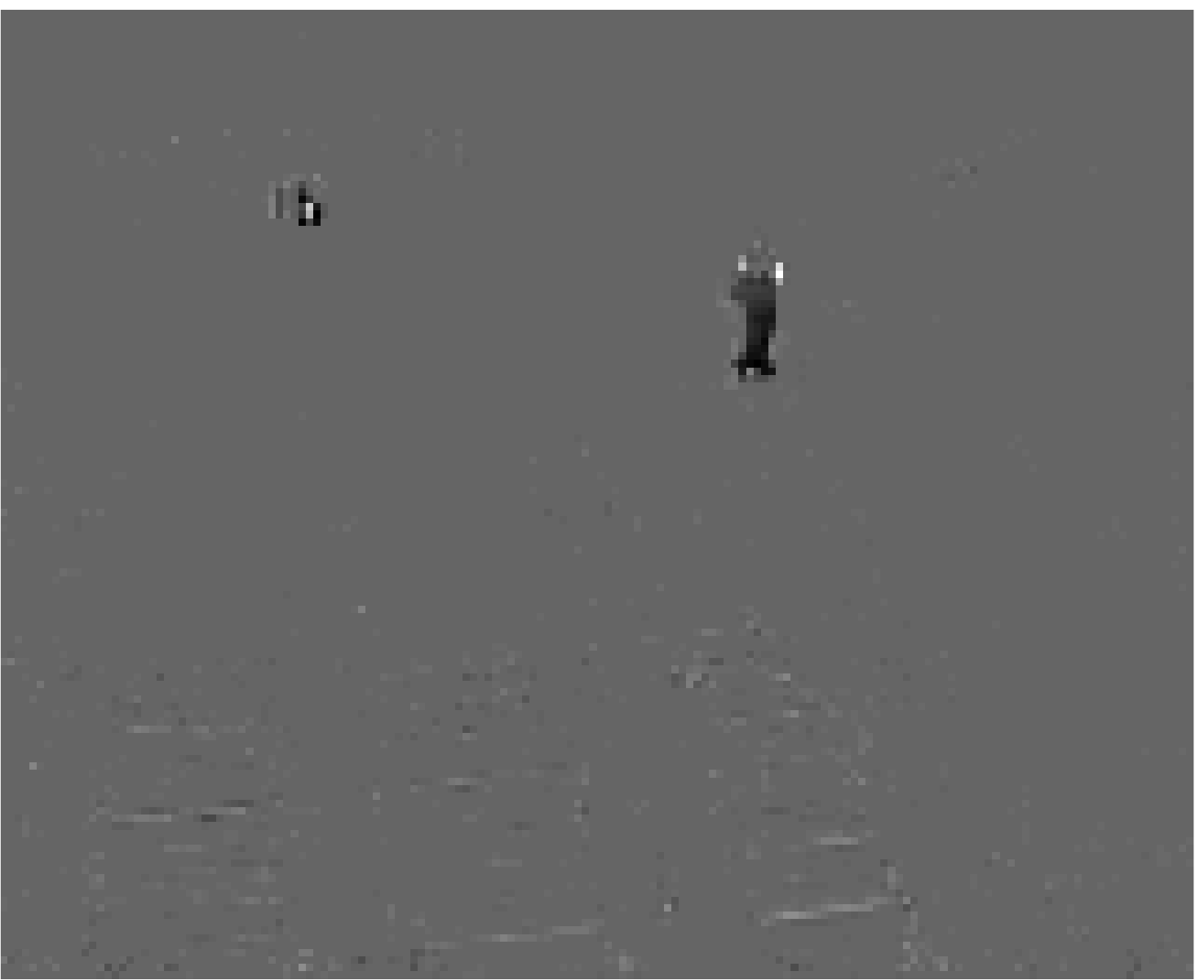}}
\caption{Frame $110$ from the movie, showing original $\obs$, low-rank $L$ and sparse $S$ components}
\label{escalator}
\end{figure}

More detailed code is available as a demo at \footnote{\url{http://cvxr.com/tfocs/demos/rpca/}}; the video data is from \footnote{\url{http://perception.i2r.a-star.edu.sg/bk_model/bk_index.html}}.
Results are shown in Fig.~\ref{escalator}. 
Although the walking person is correctly identified in the $S$ component, 
a small amount of the person appears in $L$. However, it is remarkable that 
the low-rank component mostly captures the moving escalator, 
which is a feat that most background subtraction cannot do without 
a specially targeted algorithm.

\paragraph{Comparison with PDHG}
As discussed in \S\ref{sec-literature}, the primal-dual hybrid gradient (PDHG) method is similar to the TFOCS algorithm. In TFOCS, one controls the value of $\mu$ and then runs the proximal point algorithm, and the sub-problem is solved by FISTA (or variants) that use linesearch techniques and therefore do not need a stepsize.
In PDHG, there is no line search, but there are two stepsizes $\tau$ and $\sigma$ which are linked in the fashion $\tau \sigma < 1/\Lipschitz^2$. Larger stepsizes generally lead to better performance, so by insisting $\tau \sigma = 0.99/\Lipschitz^2$, there is only one effective parameter choice.

\begin{figure}[ht]
\centering
\subfigure[TFOCS]{\includegraphics[width=.4\textwidth]{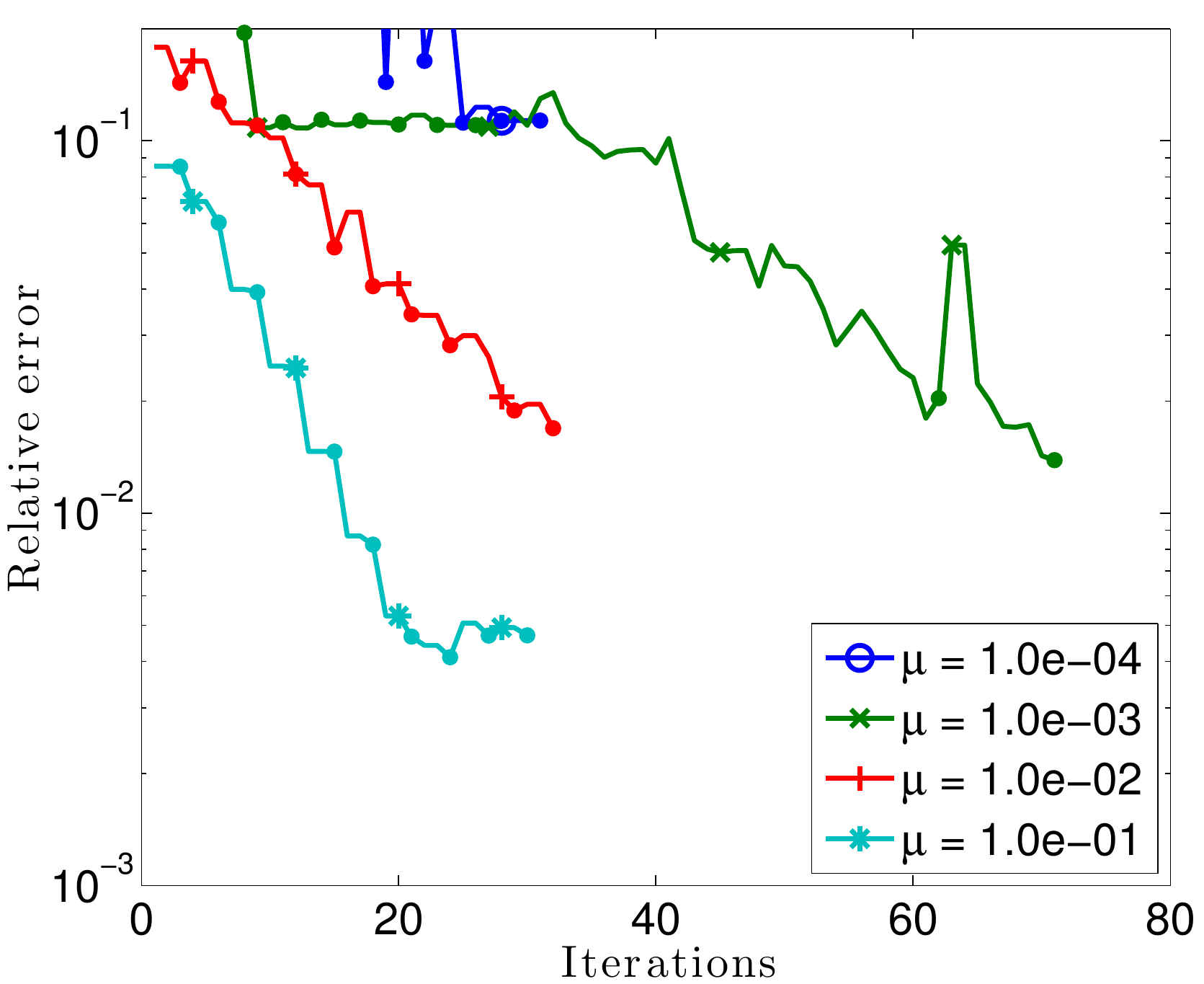}}
\subfigure[PDHG]{\includegraphics[width=.4\textwidth]{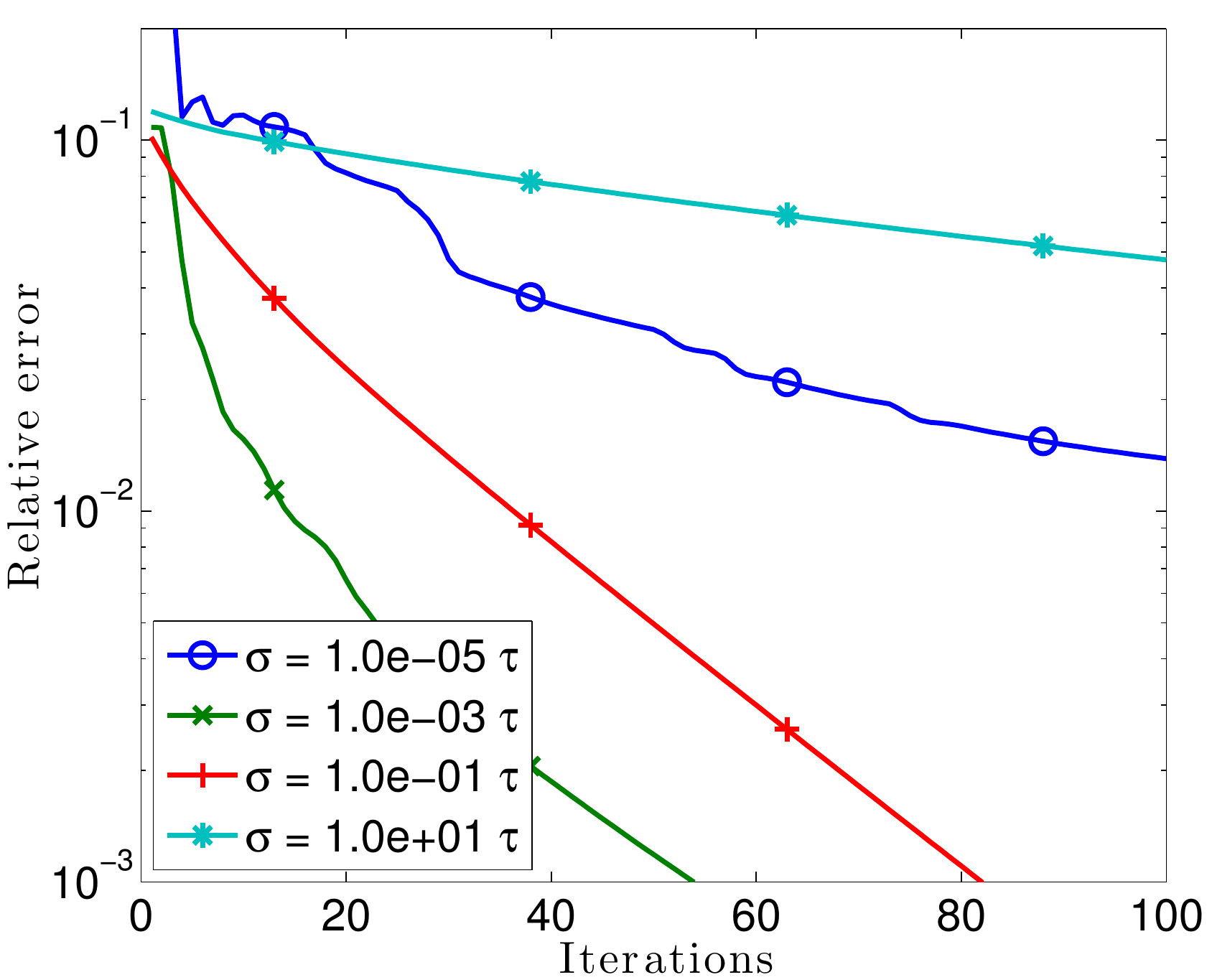}}
\caption{Convergence plots on the elevator data, for various parameter values. The small dots in the TFOCS plot show where the proximal point algorithm took another outer step.}
\label{fig-PDHG1}
\end{figure}

\begin{figure}[ht]
\centering
\subfigure[TFOCS]{\includegraphics[width=.25\textwidth]{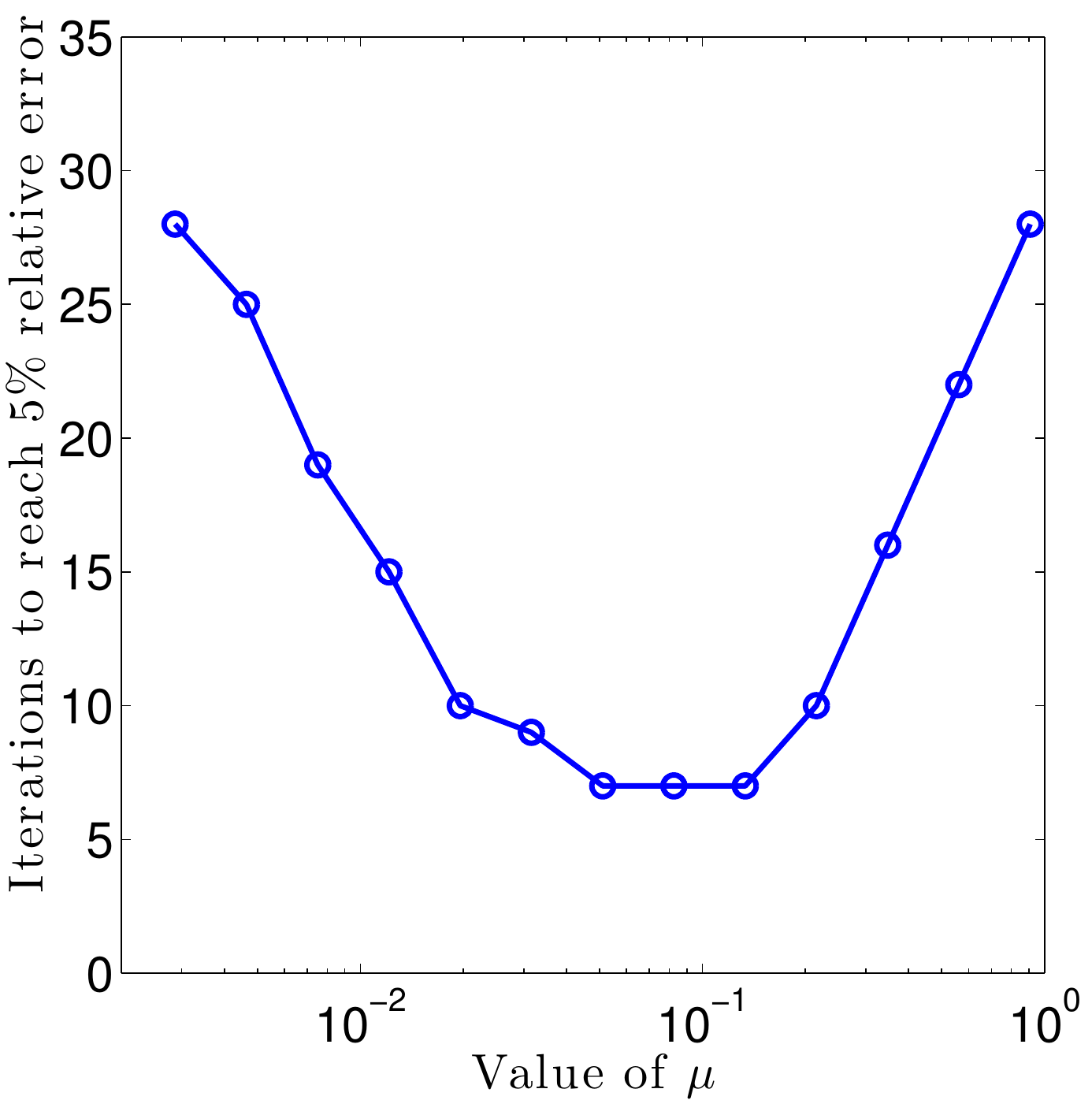}}
\subfigure[PDHG]{\includegraphics[width=.25\textwidth]{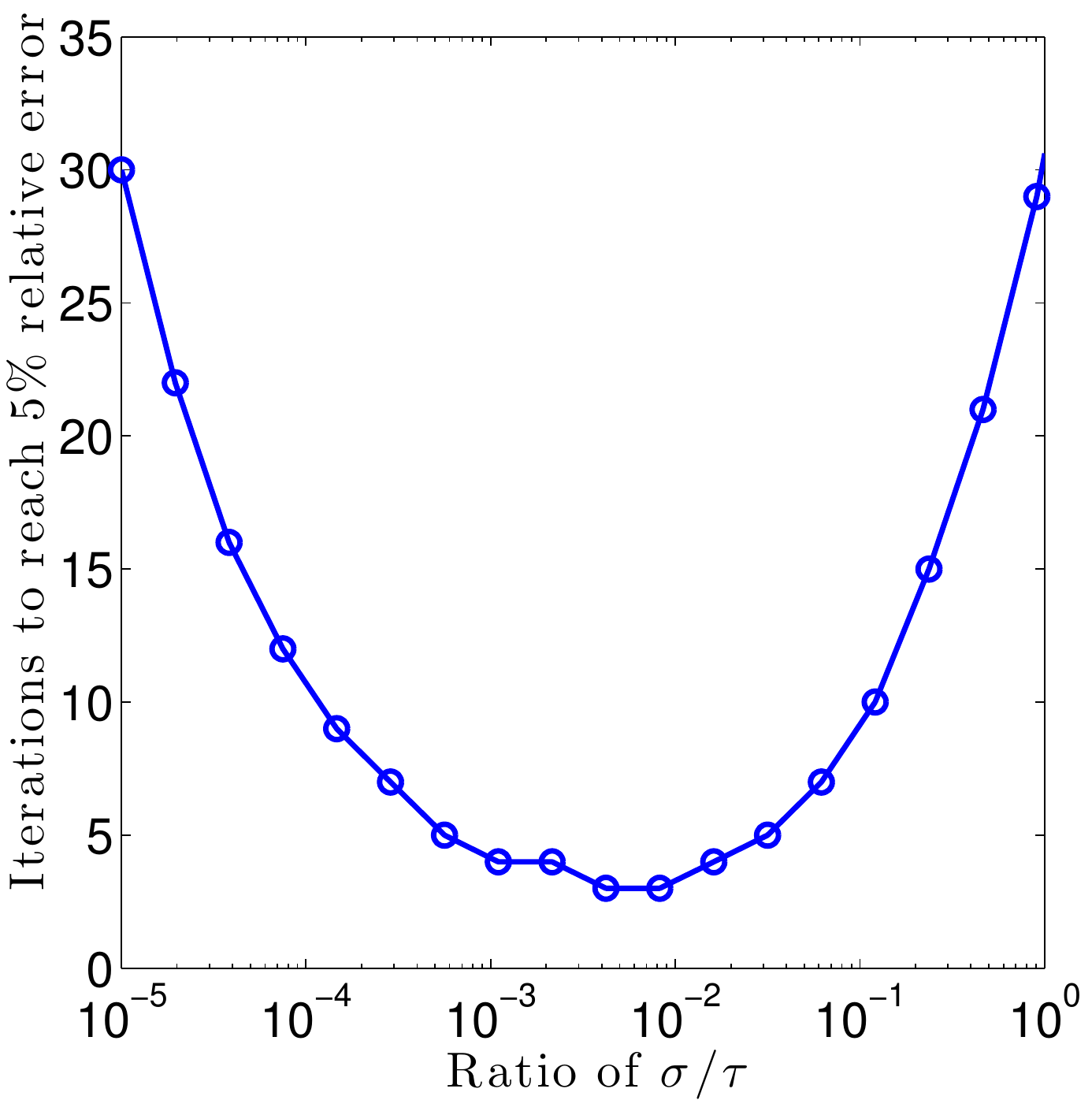}}
\caption{Number of iterations to reach a fixed tolerance, as function of parameter value}
\label{fig-PDHG2}
\end{figure}

Using a version of the same escalator film discussed in the previous section, we compare PDHG and TFOCS on \eqref{eq:sum-SPCP}. In TFOCS, we encode this with
\begin{verbatim}
dualProx = prox_l2(epsilon);
\end{verbatim}
The parameters $\epsilon$ and $\lambda$ were chosen by running the quasi-Newton solver on \eqref{eq:lag-SPCP} and tuning by hand until results looked acceptable. Running the special purpose solver gives a very accurate reference answer. From the solution to \eqref{eq:lag-SPCP}, one can infer $\epsilon$ and $\lambdaSum$.

Neither method is specific for RPCA problems, so we do not expect cutting-edge performance, but we do see reliable performance, and the ability to adapt to variations in the model. We focus on parameter selection. Both methods perform roughly equally, and both are strongly dependent on the parameter choice. A major weakness of all current methods is lack of guidance for choosing parameters in practice; the effort of \citet{PockPrecond} to find good values resulted in mixed success. The software TFOCS automatically rescales variables in order to make all $\linear_i$ terms have the same spectral norm, which has a small beneficial effect.

Fig.~\ref{fig-PDHG1} shows the decay of the relative error $\sqrt{ \|L-L_0\|_F^2 + \|S-S_0\|_F^2 }/\sqrt{ \|L_0\|_F^2 + \|S_0\|_F^2 }$ where $(L_0,S_0)$ is the accurate reference solution computed via the quasi-Newton algorithm. TFOCS has the advantage that the sub-problems can use a fast solver with a good linesearch, but the disadvantage that with the two levels of iterations, the inner iteration must be terminated at the right time. If it is stopped too early, the solution is not accurate enough, while if it is stopped too late, the algorithm wastes time on finding a very precise answer to a useless intermediate problem. 

Fig.~\ref{fig-PDHG2} shows the number of iterations required to reach a fixed tolerance. Confirming the behavior in the previous figure, it is clear that convergence can be rapid for good parameter choices, and slow for poor parameter choices.

\subsection{Numerical results for quasi-Newton algorithm}
\label{sec:numerics}
We compare new algorithms and formulations to PSPG~\citep{AybatRPCA}, NSA~\citep{Aybat2013}, and 
ASALM~\citep{ASALM}\footnote{PSPG, NSA and ASALM available from the experiment package at \url{http://www2.ie.psu.edu/aybat/codes.html}}. 
We modified the other software as needed for testing purposes. PSPG, NSA and ASALM all solve \eqref{eq:sum-SPCP}, 
but ASALM has another variant which solves \eqref{eq:lag-SPCP} so we test this as well. 
All three programs also use versions of PROPACK from \citet{BeckerPROPACK} to compute partial SVDs.
We measure error as a function of time, since cost of a single iteration can vary among the solvers. 
To fairly compare all the algorithms in the simulated experiments, we measure the (relative) error of a trial solution $(L,S)$ 
to a reference solution $(L^\star,S^\star)$ as $\|L-L^\star\|_F/\|L^\star\|_F +  \|S-S^\star\|_F/\|S^\star\|_F$. 
Time to compute this error is accounted for (so does not factor into the comparisons). 
Finally, since stopping conditions are solver dependent, we show plots of error vs time. 
All tests are done in Matlab and the dominant computational time was due to matrix multiplications for all algorithms; 
all code was run in the same quad-core $1.6$~GHz i7 computer.

For our implementations of the \eqref{eq:max-SPCP-flip}, \eqref{eq:sum-SPCP-flip} and \eqref{eq:lag-SPCP}, we use a randomized SVD~\citep{halko2011finding}. 
Since the number of singular values needed is not known in advance, the partial SVD may be called several times (the same is true for PSPG, NSA and ASALM). 
Our code limits the number of singular values on the first two iterations in order to speed up calculation without affecting convergence.
Since the projection required by~\eqref{eq:sum-SPCP-flip} makes a partial SVD difficult, 
so we use Matlab's dense \texttt{SVD} routine.

\subsubsection{Synthetic test with exponential noise}
We first provide a test with generated data. The observations $\obs\in\R^{m\times n}$ 
with $m=400$ and $n=500$ were created by first sampling a rank $20$ matrix $\obs_0$ 
with random singular vectors (i.e.,~from the Haar measure) and singular values drawn from a uniform distribution 
with mean $0.1$, and then adding exponential random noise to the entire mtrix (with mean equal to one tenth the median absolute value of the entries of $\obs_0$). 
This exponential noise, which has a longer tail than Gaussian noise, is expected to be captured partly by the $S$ term and partly by the error term $\|L+S-\obs\|_F$.

Given $\obs$, the reference solution $(L^\star,S^\star)$ was generated by solving \eqref{eq:lag-SPCP} to very high accuracy; 
the values $\lambdaL=0.25$ and $\lambdaS=10^{-2}$ were picked by hand tuning $(\lambdaL,\lambdaS)$ to find a value such that both $L^\star$ and $S^\star$ are non-zero.
The advantage to solving \eqref{eq:lag-SPCP} is that knowledge of $(L^\star,S^\star,\lambdaL,\lambdaS)$ 
allows us to generate the parameters for all the other variants, and hence we can test different problem formulations.

With these parameters, $L^\star$ was rank 17 with nuclear norm $6.754$, 
$S^\star$ had 54 non-zero entries (most of them positive) with $\ell_1$ norm $0.045$, 
the normalized residual was $\|L^\star+S^\star-\obs\|_F/\|\obs\|_F=0.385$, and $\eps=1.1086$, 
$\lambdaSum=0.04$, $\lambdaMax=150.0593$, $\tauSum=6.7558$ and $\tauMax=6.7540$.

\begin{figure}[ht]
\centering
\includegraphics[width=.6\textwidth]{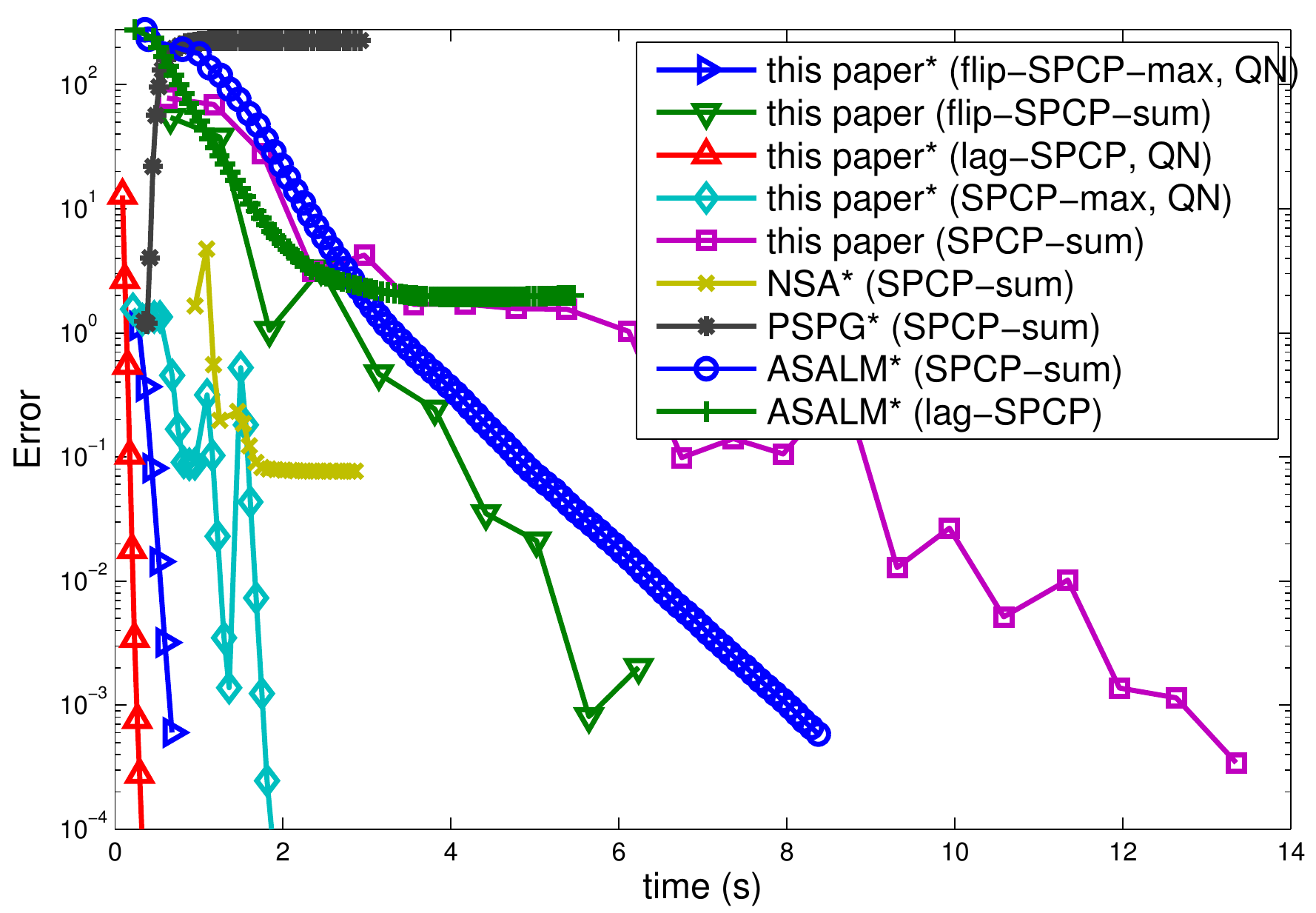}
\caption{The exponential noise test. The asterisk in the legend means the method uses a fast SVD (i.e., randomized SVD). }
\label{P1C3fig:1}
\end{figure}

Results are shown in Fig.~\ref{P1C3fig:1}. Our methods for \eqref{eq:max-SPCP-flip} and \eqref{eq:lag-SPCP} are extremely fast, 
because the simple nature of these formulations allows the quasi-Newton acceleration scheme of Section~\ref{sec:QN}.
In turn, since our method for solving~\eqref{eq:max-SPCP} uses the variational framework of Section~\ref{sec:variational} 
to solve a sequence of~\eqref{eq:max-SPCP-flip} problems, it is also competitive (shown in cyan in Figure~\ref{P1C3fig:1}).  
The jumps are due to re-starting the sub-problem solver with a new value of $\tau$, generated according to~\eqref{eq:newton}. 

Our proximal gradient method for \eqref{eq:sum-SPCP-flip}, which makes use of the projection in Lemma~\ref{lemma:jointProjection}, 
converges more slowly, since it is not easy to accelerate with the quasi-Newton scheme due to variable coupling, 
and it does not make use of fast SVDs. Our solver for~\eqref{eq:sum-SPCP}, which depends on a sequence of 
problems~\eqref{eq:sum-SPCP-flip},  converges slowly. 

The ASALM performs reasonably well, which was unexpected since it was shown to be worse than NSA and PSPG in \citep{AybatRPCA,Aybat2013}. 
The PSPG solver converges to the wrong answer, most likely due to a bad choice of the smoothing parameter $\mu$; 
we tried choosing several different values other than the default but did not see improvement for this test (for other tests, not shown, tweaking $\mu$ helped significantly). 
The NSA solver reaches moderate error quickly but stalls before finding a highly accurate solution.

\subsubsection{Synthetic test from \citet{Aybat2013}}
We show some tests from the test setup of \citet{Aybat2013} in the $m=n=1500$ case. The default setting of $\lambdaSum=1/\sqrt{\max(m,n)}$ was used, and then the NSA solver was run to high accuracy to obtain a reference solution $(L^\star,S^\star)$.  From the knowledge of $(L^\star,S^\star,\lambdaSum)$, 
one can generate $\lambdaMax,\tauSum,\tauMax,\eps$, but not $\lambdaS$ and $\lambdaL$, 
and hence we did not test the solvers for \eqref{eq:lag-SPCP} in this experiment. 
The data were generated as $\obs=L_0+S_0+Z_0$, where $L_0$ was sampled by multiplication of $m \times r$ and $r \times n$ normal Gaussian matrices, 
$S_0$ had $p$ randomly chosen entries uniformly distributed within $[-100,100]$, and $Z_0$ was white noise chosen to give a SNR of $45$~dB. 
We show three tests that vary the rank from $\{0.05,0.1\}\cdot \min(m,n)$ 
and the sparsity ranging from $p=\{0.05,0.1\}\cdot mn$.
Unlike \citet{Aybat2013}, who report error in terms of a true noiseless signal $(L_0,S_0)$, 
we report the optimization error relative to $(L^\star,S^\star)$.

For the first test (with $r=75$ and $p=0.05\cdot mn$), $L^\star$ had rank $786$ and nuclear norm $111363.9$; 
$S^\star$ had $75.49\%$ of its elements nonzero and $\ell_1$ norm $ 5720399.4$, 
and $\|L^\star+S^\star-\obs\|_F/\|\obs\|_F=1.5\cdot 10^{-4}$. 
The other parameters were $\eps=3.5068$, $\lambdaSum=0.0258$, $\lambdaMax=0.0195$, $\tauSum=2.5906\cdot 10^{5}$ and $\tauMax=1.1136\cdot 10^{5}$.
An interesting feature of this test is that while $L_0$ is low-rank, $L^\star$ is nearly low-rank but with a small tail of significant singular values until number $786$. 
We expect methods to converge quickly to low-accuracy, 
and then slow down as they try to find a highly-accurate larger rank solution. 

\begin{figure}[ht]
\centering
\includegraphics[width=.5\textwidth]{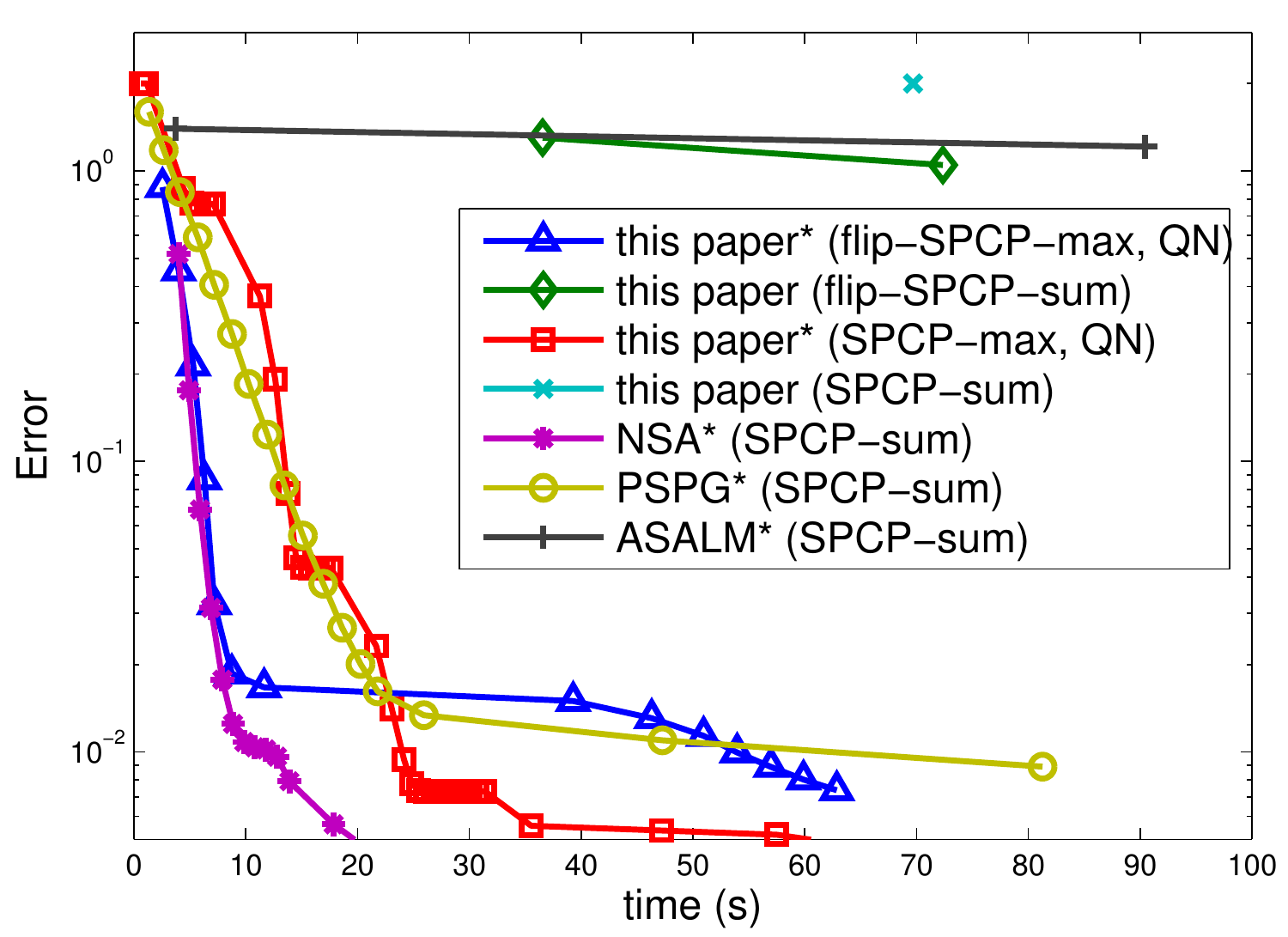}
\caption{The $1500\times 1500$ synthetic noise test. }
\label{P1C3fig:2}
\end{figure}

The results are shown in Fig.~\ref{P1C3fig:2}. Errors barely dip below $0.01$ (for comparison, an error of $2$ is achieved by setting $L=S=0$).
The NSA and PSPG solvers do quite well. In contrast to the previous test, ASALM does poorly. 
Our methods for \eqref{eq:sum-SPCP-flip}, and hence \eqref{eq:sum-SPCP}, are not competitive, since they use dense SVDs. 
We imposed a time-limit of about one minute, so these methods only manage a single iteration or two. 
Our quasi-Newton method for \eqref{eq:max-SPCP-flip} does well initially, then takes a long time due to a long partial SVD computation. 
Interestingly, \eqref{eq:max-SPCP} does better than pure \eqref{eq:max-SPCP-flip}. 
One possible explanation is that it chooses a fortuitous sequence of $\tau$ values, 
for which the corresponding \eqref{eq:max-SPCP-flip} subproblems become increasingly hard, 
and therefore benefit from the warm-start of the solution of the easier previous problem.
This is consistent with empirical observations regarding continuation techniques, see e.g.,~\citep{vandenberg2008probing,Wright2009}.

\begin{figure}[ht]
\centering
\includegraphics[width=.5\textwidth]{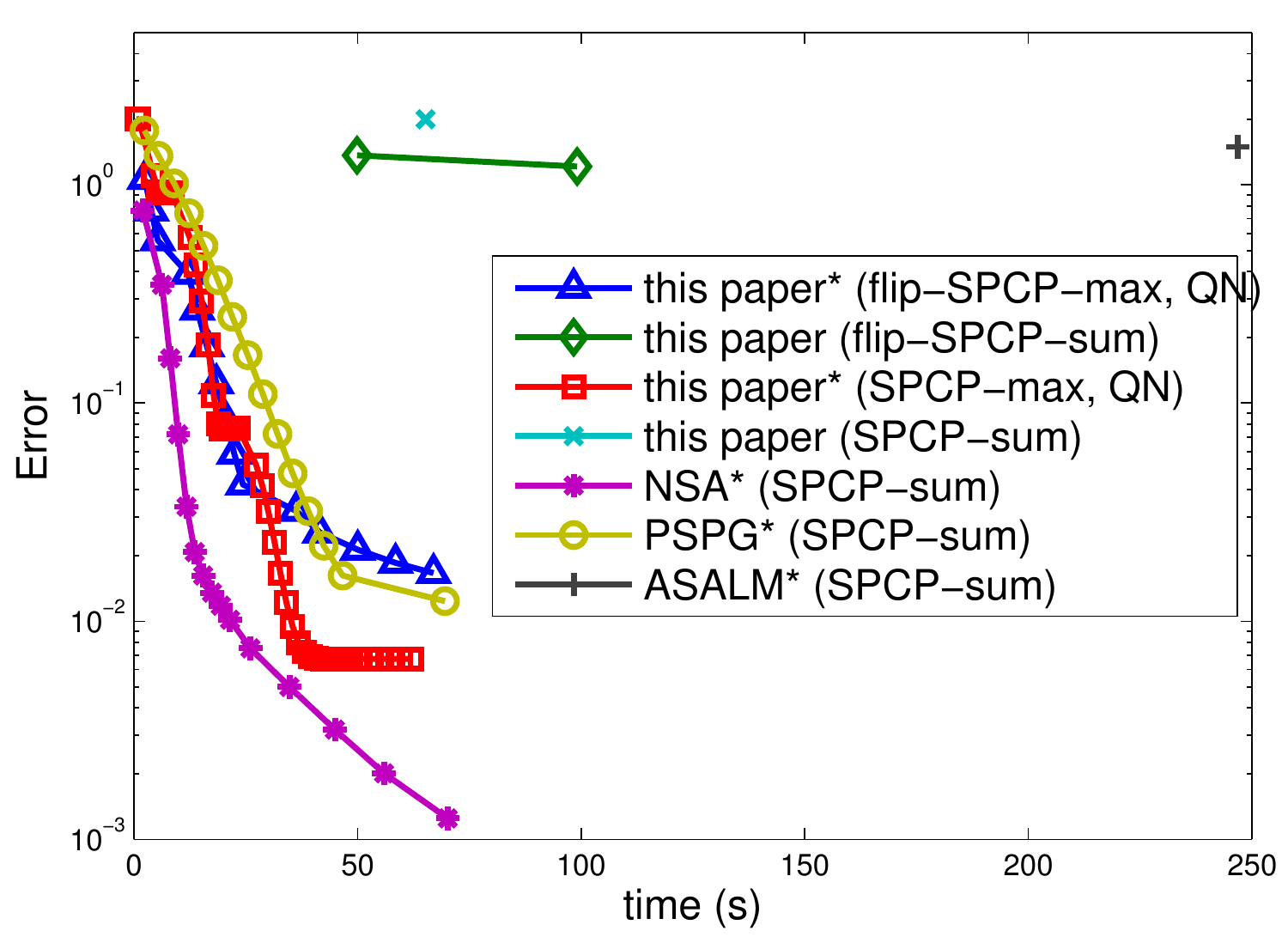}
\caption{Second $1500\times 1500$ synthetic noise test.}
\label{P1C3fig:3}
\end{figure}

Fig.~\ref{P1C3fig:3} is the same test but with $r=150$ and $p=0.1\cdot m n$, and the conclusions are largely similar.

\subsubsection{Cloud removal}
\begin{figure*}[ht]
\centering
\includegraphics[width=.7\textwidth]{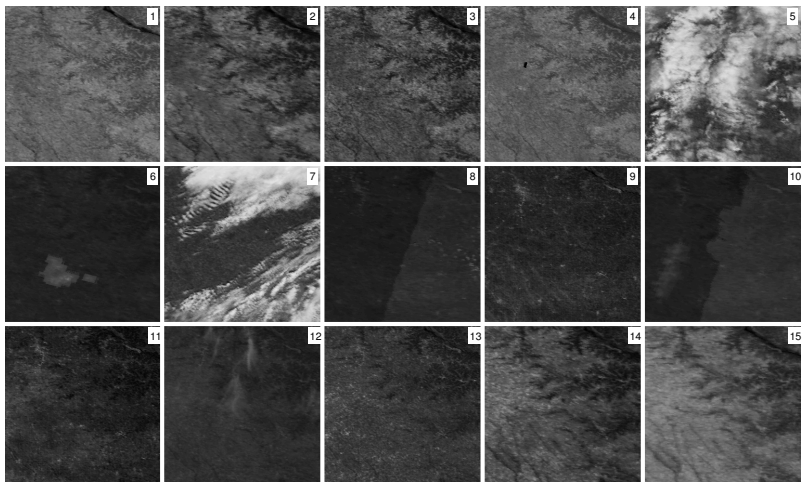}
\caption{Satellite photos of the same location on different days.}
\label{P1C3fig:4}
\end{figure*}

\begin{figure}[t]
    \centering
    \includegraphics[width=.5\textwidth]{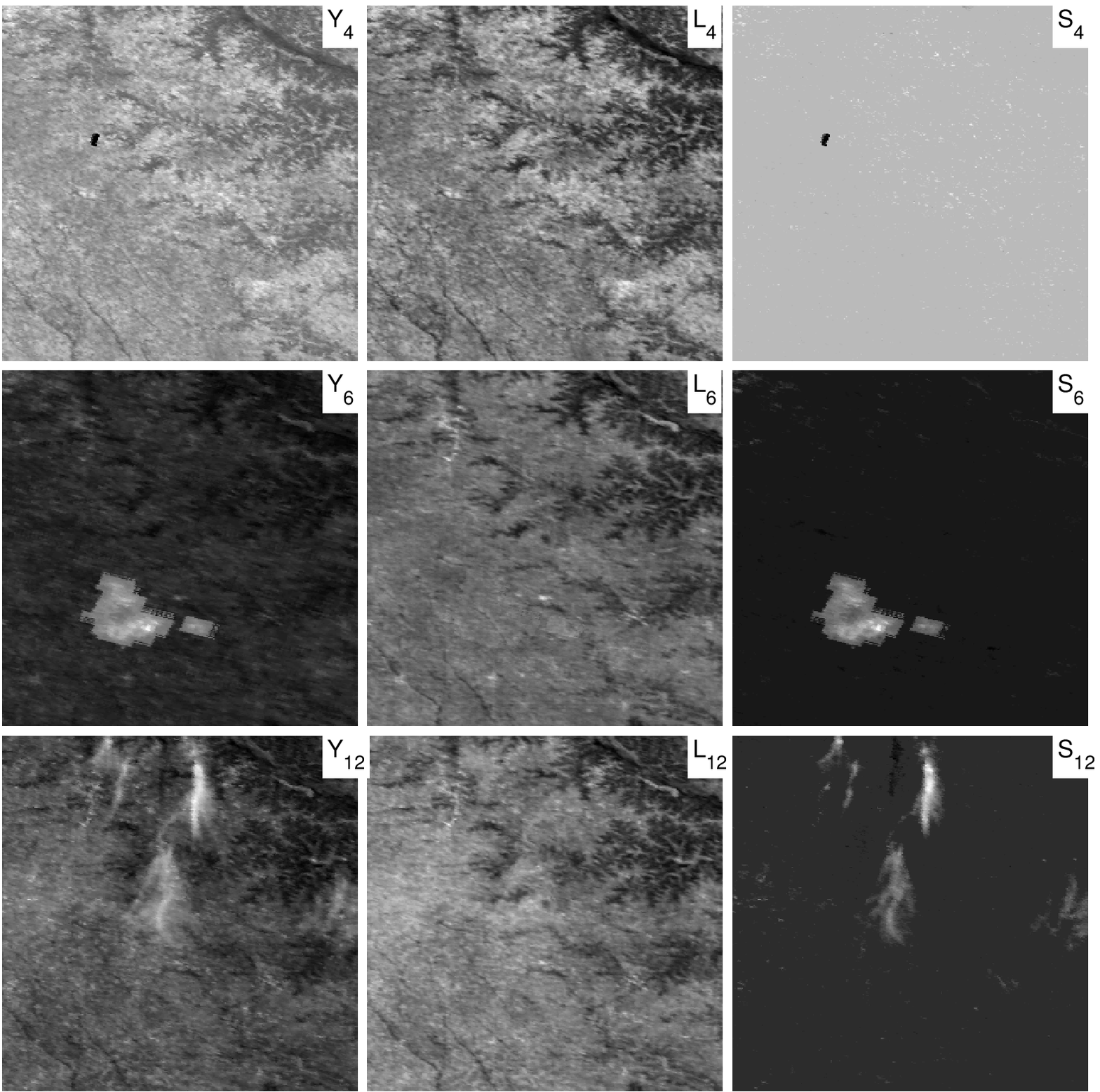}
    \caption{Showing frames 4, 5 and 12. Leftmost column is original data, middle column is low-rank term of the solution, and right column is sparse term of the solution. Data have been processed slightly to enhance contrast for viewing. }
    \label{P1C3fig:5}
\end{figure}

Figure~\ref{P1C3fig:4} shows 15 images of size $300 \times 300$ that come from the MODIS satellite\footnote{Publicly available at \url{http://ladsweb.nascom.nasa.gov/}}, 
after some transformations to turn images from different spectral bands into one grayscale images. 
Each image is a photo of the same rural location but at different points in time over the course of a few months. 
The background changes slowly and the variability is due to changes in vegetation, snow cover, and different reflectance. 
There are also outlying sources of error, mainly due to clouds (e.g., major clouds in frames 5 and 7, smaller clouds in frames 9, 11 and 12), 
as well as artifacts of the CCD camera on the satellite (frame 4 and 6) and issues stitching together photos of the same scene (the lines in frames 8 and 10).

There are many applications for clean satellite imagery, so removing the outlying error is of great practical importance. 
Because of slow changing background and sparse errors, we can model the problem using the robust PCA approach. 
We use the \eqref{eq:max-SPCP-flip} version due to its speed, and pick parameters $(\lambdaMax,\tauMax)$ by using a Nelder-Mead simplex search. 
For an error metric to use in the parameter tuning, we remove frame $1$ from the data set (call it $y_1$) and set $\obs$ to be frames 2--15. 
From this training data $\obs$, the algorithm generates $L$ and $S$. Since $L$ is a $300^2 \times 14$ matrix, it has far from full column span. 
Thus our error is the distance of $y_1$ from the span of $L$, i.e., $\|y_1 - \proj_{\text{span}(L)}(y_1)\|_2$.

Our method takes about 11 iterations and 5 seconds, and uses a dense SVD instead of the randomized method due to the extremely high aspect ratio of the matrix: the matrix is $15 \times 300^2$, and the cost of a dense SVD is linear in the large dimension, so the computational burden is not large. 
Some results of the obtained $(L,S)$ outputs are in Fig.~\ref{P1C3fig:5}, where one can see that some of the anomalies in the original data frames $\obs$ 
are picked up by the $S$ term and removed from the $L$ term. 
Frame 4 has what appears to be a camera pixel error;  frame 6 has another artificial error (that is, caused by the camera and not the scene); and frame 12 has cloud cover.

\subsubsection{Analysis of brain activity in the zebra fish}
Recent work by \citet{ZebraFish} has produced video recordings of brain activity, in vivo, of zebra fish. These datasets are used to confirm scientific theories about the inner-working of the brain as well as to discover unexpected connections. Ultimately the goal is to discover causal, not just correlated, relationships. PCA on these datasets is one of the standard tools used by biologists in order to uncover correlations.

Using a public video of the dataset, we focus on a single 2D slice, sub-sampled spatially (and perhaps with video compression artifacts). We use \eqref{eq:max-SPCP} as the RPCA technique, and therefore need to estimate $\epsilon$ and $\lambdaMax$.  To find $\epsilon$, we first take the SVD of the data matrix $\obs$. The corresponding singular vales $\sigma(\obs)$ are plotted in Fig.~\ref{P1C3fig:zebra2}. This gives us an idea of the compressibility of the data. Keeping about 30 singular values explains over $99\%$ of the data, so if we look for $L$ with approximately rank $30$, then, not including the sparse term $S$ (which we expect to be very sparse), we should pick $\epsilon \simeq \sqrt{ \sum_{i=31}^{752} \sigma_i^2(\obs) }$. This value works well in practice (see Fig.~\ref{P1C3fig:zebra1}) and did not require cross-validation.

\begin{figure*}[ht]
\centering
\includegraphics[width=.5\textwidth]{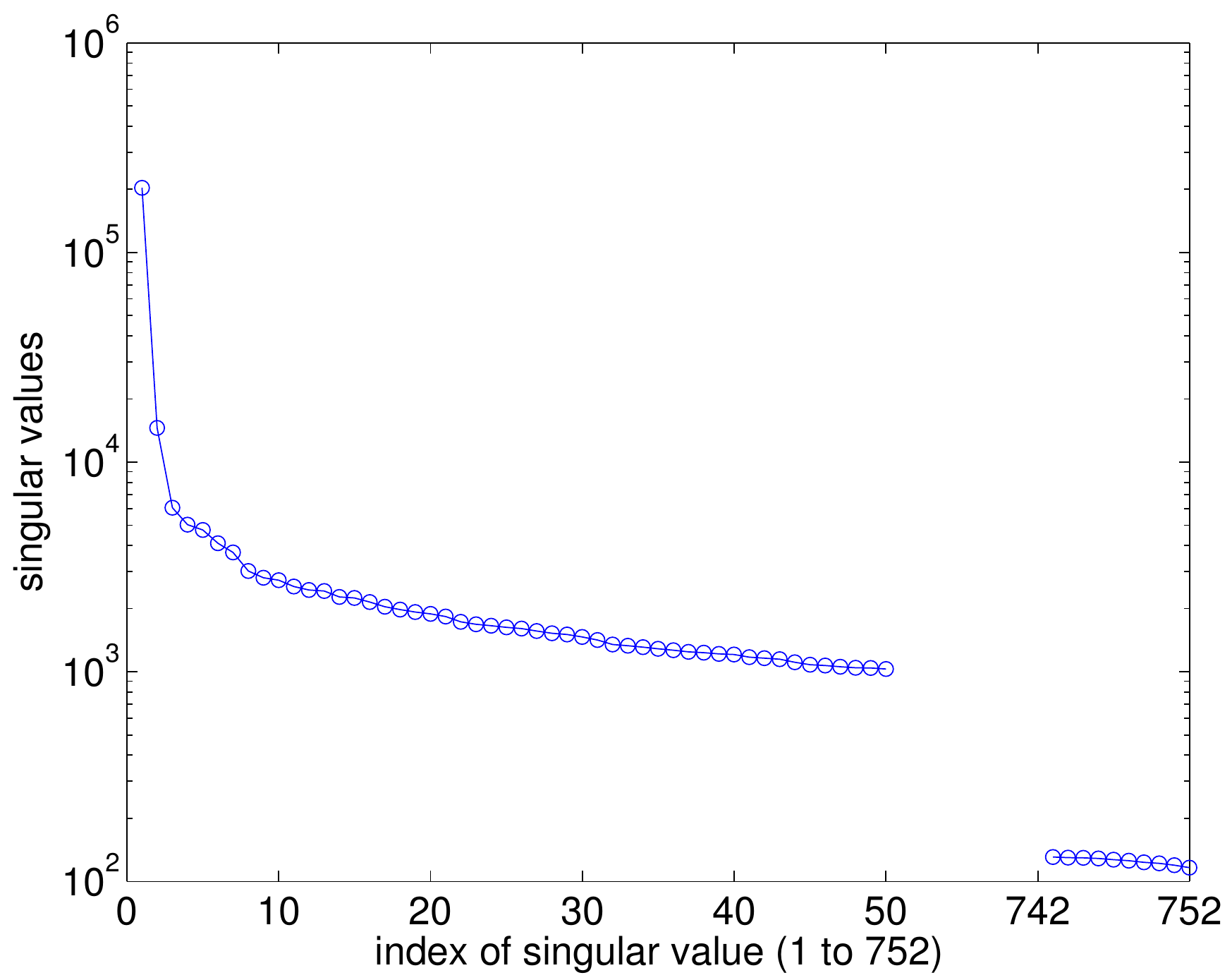}
\caption{Decay of singular values of $\obs$ for the zebra fish dataset. Singular values 51--742 are not shown.}
\label{P1C3fig:zebra2}
\end{figure*}

The $\lambdaMax$ parameter is tuned by hand, but only takes $3$ runs to find a reasonable value. This is much simpler than tuning $\lambdaMax$ and $\epsilon$ by hand simultaneously. Figure~\ref{P1C3fig:zebra1} shows the resulting top left singular vectors of $\obs$ and of $L$, as well as their difference. We see that their difference is sparse, as expected. Since these are singular vectors, not just individual frames from the movie, these sparse differences are persistent over time, and perhaps meaningful. These unpredictable locations could be caused by sensor/microscope error, or they could mean that they come from a part of the brain that is not well correlated with general brain activity. Either way, it is useful to be able to separate out this effect.

\begin{figure*}[ht]
\centering
\subfigure[Top left singular vector of $\obs$]{\includegraphics[width=.22\textwidth]{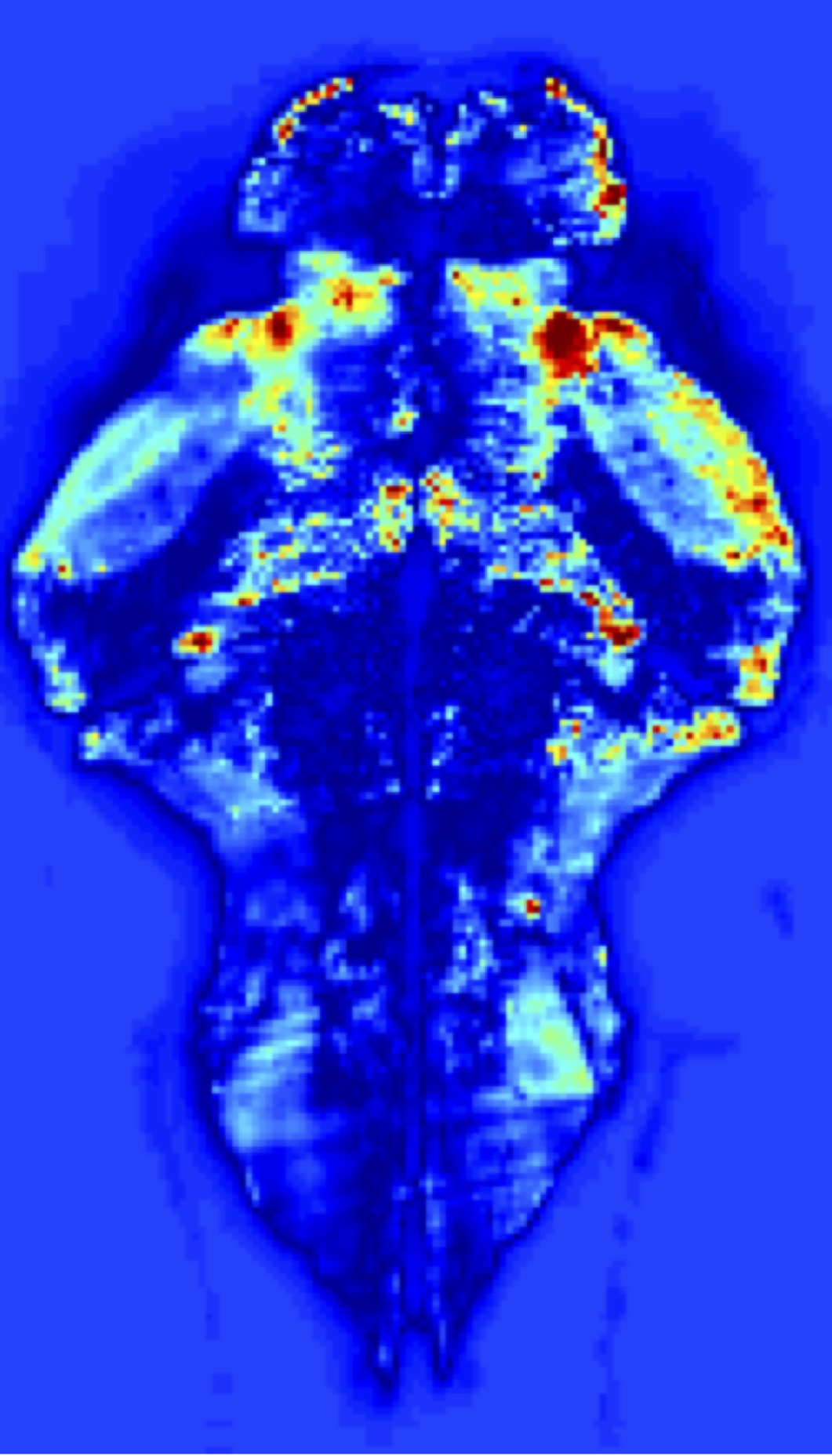}}
\subfigure[Top left singular vector of $L$.]{\includegraphics[width=.22\textwidth]{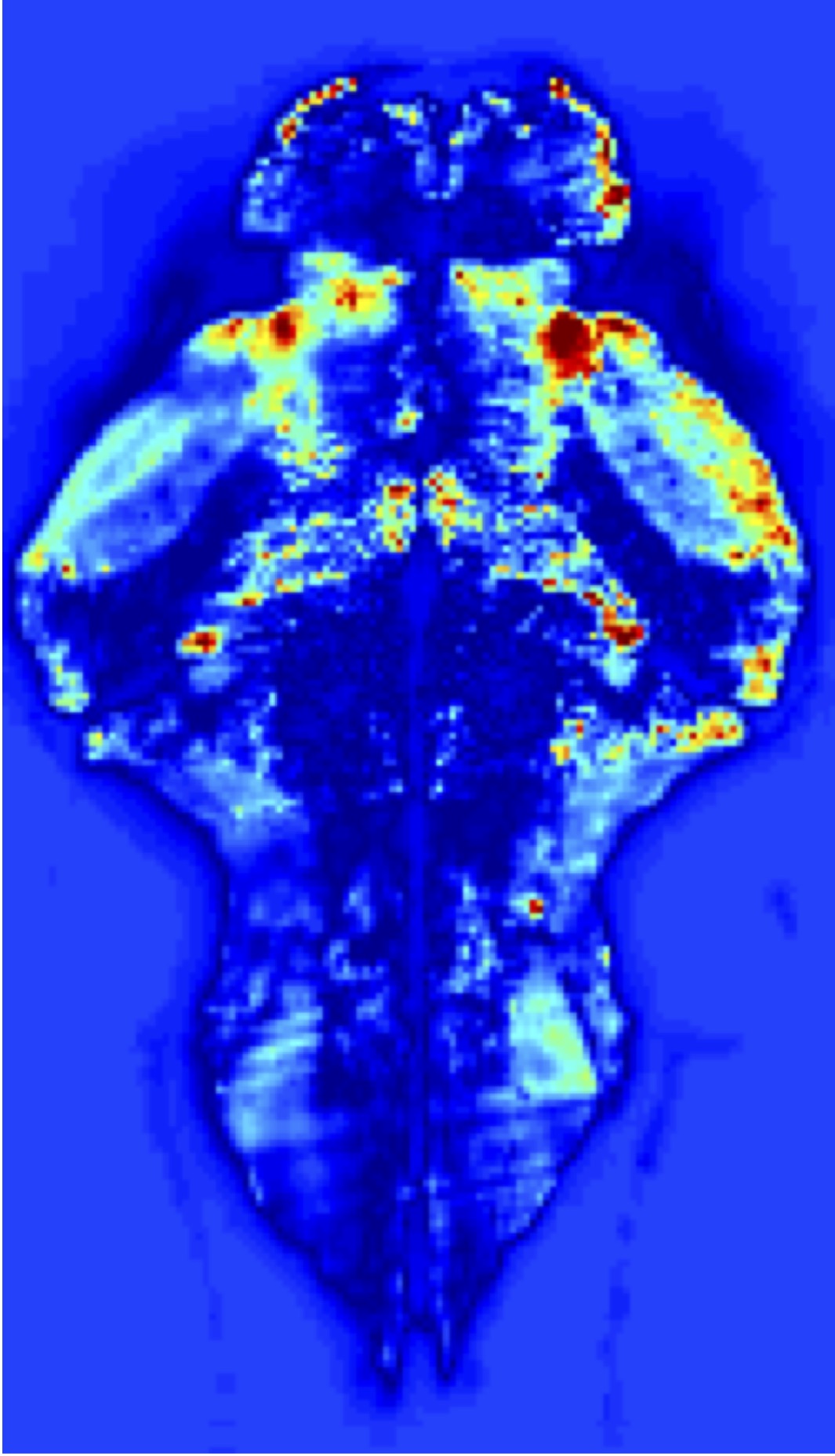}}
\subfigure[Difference of (a) and (b).]{\includegraphics[width=.22\textwidth]{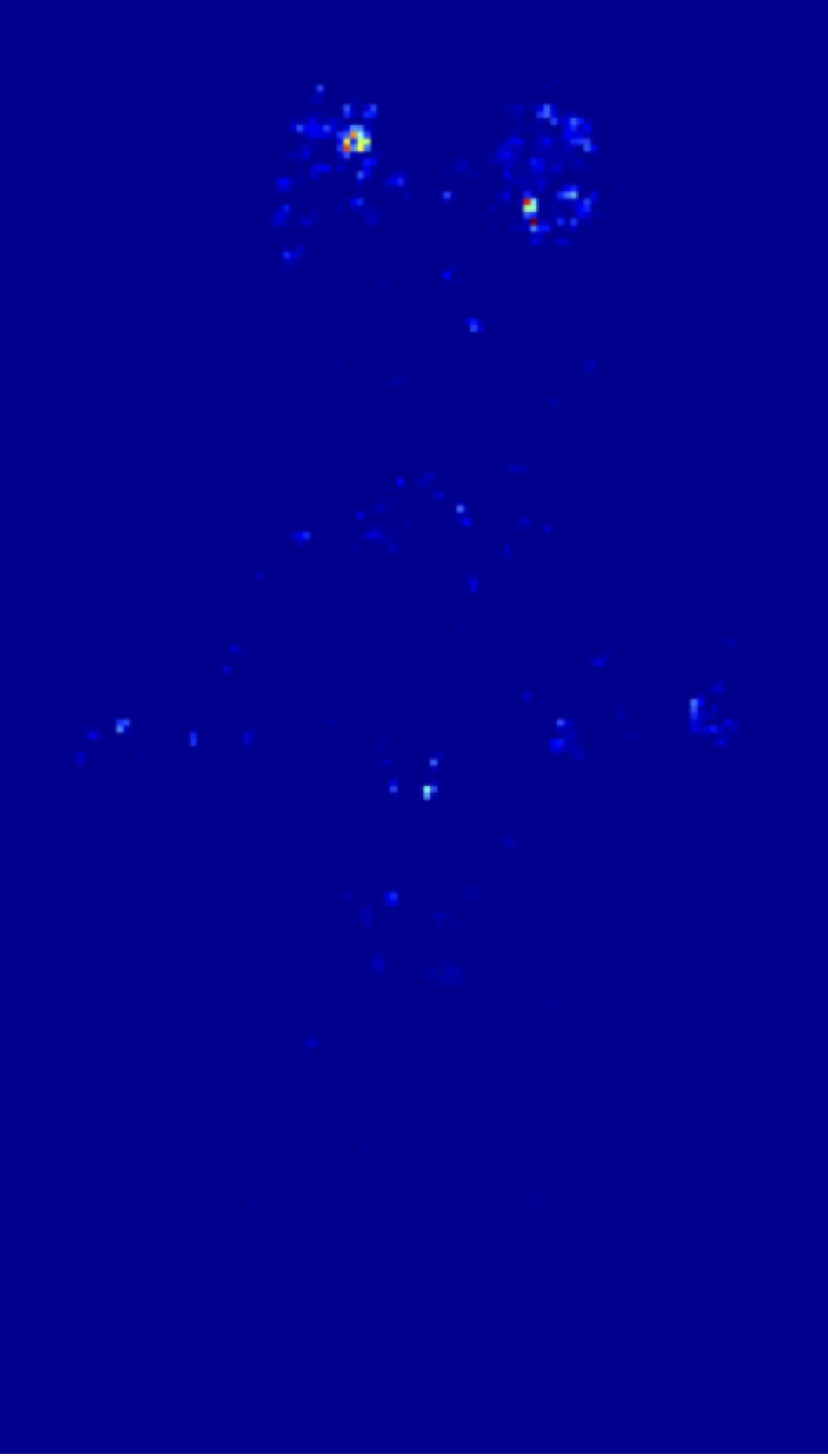}}
\caption{Top left singular vector of $\obs$ and of $L$, as well as their difference.}
\label{P1C3fig:zebra1}
\end{figure*}

\section{Conclusion}
\noindent 
We have discussed both specific algorithms for the RPCA problem, and general algorithm frameworks (``TFOCS'', and ``flipped'' variational value-function approaches) that incorporate RPCA and variants. The custom RPCA algorithm works extremely well in practice, and the process of ``flipping'' the objective has been studied rigorously, but the inner ``quasi-Newton'' algorithm lacks a rigorous convergence theory. The general algorithm TFOCS, as well as similar proposals such as the PDHG algorithm, have more established theory but lack practical guidance on setting parameters, and are in practice slower than the special purpose algorithms.  An obvious goal of future work is to either improve the analysis of these algorithms (for example, this may give insight into parameter selection), or derive new algorithms that inherent all the advantages.

The running theme of this work has been the benefits of solving variants of RPCA, in particular \eqref{eq:sum-SPCP-flip} and the new variants \eqref{eq:max-SPCP} and \eqref{eq:max-SPCP-flip}. These versions sometimes allow a good estimate of $\tau$ (or $\eps$ for \eqref{eq:max-SPCP}), thus reducing the parameter selection of the model to the single scalar $\lambdaSum$ or $\lambdaMax$.  
Our theory allows for different regularizers and data fidelity terms; using these in practice is interesting future work.

\bibliography{RPCA_Refs}

\begin{thebibliography}{101}
\providecommand{\natexlab}[1]{#1}
\providecommand{\url}[1]{\texttt{#1}}
\expandafter\ifx\csname urlstyle\endcsname\relax
  \providecommand{\doi}[1]{doi: #1}\else
  \providecommand{\doi}{doi: \begingroup \urlstyle{rm}\Url}\fi

\bibitem[Ahrens et~al.(2013)Ahrens, Orger, Robson, Li, and Keller]{ZebraFish}
M.~B. Ahrens, M.~B. Orger, D.~N. Robson, J.~M. Li, and P.~J. Keller.
\newblock Whole-brain functional imaging at cellular resolution using
  light-sheet microscopy.
\newblock \emph{Nature Methods}, 10\penalty0 (413--420), 2013.

\bibitem[Aravkin et~al.(2014{\natexlab{a}})Aravkin, Becker, Cevher, and
  Olsen]{aravkin2014variational}
A.~Aravkin, S.~Becker, V.~Cevher, and P.~Olsen.
\newblock A variational approach to stable principal component pursuit.
\newblock In \emph{Uncertainty in Artificial Intelligence}, Quebec City,
  2014{\natexlab{a}}.

\bibitem[Aravkin et~al.(2013)Aravkin, Burke, and
  Friedlander]{AravkinBurkeFriedlander:2013}
A.~Y. Aravkin, J.~V. Burke, and M.~P. Friedlander.
\newblock Variational properties of value functions.
\newblock \emph{SIAM J. Opt.}, 23\penalty0 (3):\penalty0 1689--1717, 2013.

\bibitem[Aravkin et~al.(2014{\natexlab{b}})Aravkin, Kumar, Mansour, Recht, and
  Herrmann]{Aravkin2013}
A.Y. Aravkin, R.~Kumar, H.~Mansour, B.~Recht, and F.J. Herrmann.
\newblock Fast methods for denoising matrix completion formulations, with
  applications to robust seismic data interpolation.
\newblock \emph{SIAM Journal on Scientific Computing}, 36\penalty0
  (5):\penalty0 S237--S266, 2014{\natexlab{b}}.

\bibitem[Aravkin et~al.(2014{\natexlab{c}})Aravkin, Kumar, Mansour, Recht, and
  Herrmann]{aravkin2014fast}
A.Y. Aravkin, R.~Kumar, H.~Mansour, B.~Recht, and F.J. Herrmann.
\newblock Fast methods for denoising matrix completion formulations, with
  applications to robust seismic data interpolation.
\newblock \emph{SIAM Journal on Scientific Computing}, 36\penalty0
  (5):\penalty0 S237--S266, 2014{\natexlab{c}}.

\bibitem[Aravkin et~al.(2016)Aravkin, Burke, Drusvyatskiy, Friedlander, and
  Roy]{aravkin2016level}
A.Y. Aravkin, J.V. Burke, D.~Drusvyatskiy, M.P. Friedlander, and S.~Roy.
\newblock Level-set methods for convex optimization.
\newblock \emph{arXiv preprint arXiv:1602.01506}, 2016.

\bibitem[Aybat et~al.(2013)Aybat, Goldfarb, and Ma]{AybatRPCA}
N.~Aybat, D.~Goldfarb, and S.~Ma.
\newblock Efficient algorithms for robust and stable principal component
  pursuit.
\newblock \emph{Computational Optimization and Applications, {\em accepted}},
  2013.

\bibitem[Aybat and Iyengar(2013)]{Aybat2013}
N.~S. Aybat and G.~Iyengar.
\newblock A fast first-order method for stable principal component pursuit.
\newblock \emph{\url{http://arxiv.org/abs/1309.6553}}, 2013.

\bibitem[Bauschke and Combettes(2011)]{Bauschke2011}
H.H. Bauschke and P.~Combettes.
\newblock \emph{Convex analysis and monotone operators theory in {H}ilbert
  spaces}.
\newblock Springer-Verlag, 2011.

\bibitem[Beck and Teboulle(2009{\natexlab{a}})]{BecTeb09}
A.~Beck and M.~Teboulle.
\newblock {A Fast Iterative Shrinkage-Thresholding Algorithm for Linear Inverse
  Problems}.
\newblock \emph{SIAM J. Imaging Sciences}, 2\penalty0 (1):\penalty0 183--202,
  January 2009{\natexlab{a}}.

\bibitem[Beck and Teboulle(2009{\natexlab{b}})]{BecTeb:09}
A.~Beck and M.~Teboulle.
\newblock A fast iterative shrinkage-thresholding algorithm for linear inverse
  problems.
\newblock \emph{SIAM J. Imaging Sci.}, 2\penalty0 (1):\penalty0 183--202,
  2009{\natexlab{b}}.
\newblock ISSN 1936-4954.
\newblock \doi{10.1137/080716542}.
\newblock URL \url{http://dx.doi.org/10.1137/080716542}.

\bibitem[Beck and Teboulle(2014)]{BeckTeboulle14}
A.~Beck and M.~Teboulle.
\newblock A fast dual proximal gradient algorithm for convex minimization and
  applications.
\newblock \emph{Operations Research Letters}, 42\penalty0 (1):\penalty0 1--6,
  2014.
\newblock ISSN 0167-6377.
\newblock \doi{http://dx.doi.org/10.1016/j.orl.2013.10.007}.
\newblock URL
  \url{http://www.sciencedirect.com/science/article/pii/S0167637713001454}.

\bibitem[Becker(2011)]{BeckerThesis}
S.~Becker.
\newblock \emph{Practical Compressed Sensing: modern data acquisition and
  signal processing}.
\newblock PhD thesis, California Institute of Technology, Pasadena, CA, 2011.

\bibitem[Becker and Cand\`es(2008)]{BeckerPROPACK}
S.~Becker and E.~Cand\`es.
\newblock Singlar value thresholding toolbox, 2008.
\newblock Available from \url{http://svt.stanford.edu/}.

\bibitem[Becker and Combettes(2014)]{infConv_paper}
S.~Becker and P.~L. Combettes.
\newblock An algorithm for splitting parallel sums of linearly composed
  monotone operators, with applications to signal recovery.
\newblock \emph{J. Nonlinear and Convex Analysis}, 15\penalty0 (1):\penalty0
  137--159, 2014.

\bibitem[Becker et~al.(2011)Becker, Cand\`es, and Grant]{TFOCS}
S.~Becker, E.~J. Cand\`es, and M.~Grant.
\newblock Templates for convex cone problems with applications to sparse signal
  recovery.
\newblock \emph{Math. Prog. Comp.}, 3\penalty0 (3):\penalty0 165--218, 2011.

\bibitem[Bertsekas(1975)]{Bertsekas75}
D.~P. Bertsekas.
\newblock Necessary and sufficient conditions for a penalty method to be exact.
\newblock \emph{Math. Program.}, 9:\penalty0 87--99, 1975.

\bibitem[Bertsekas et~al.(2003)Bertsekas, Nedi\'c, and Ozdaglar]{BertsekasBook}
D.~P. Bertsekas, A.~Nedi\'c, and A.~E. Ozdaglar.
\newblock \emph{Convex Analysis and Optimization}.
\newblock Athena Scientific, 2003.

\bibitem[Bo{\c{t}} and Csetnek(2014)]{botFBF}
R.~I. Bo{\c{t}} and E.~R. Csetnek.
\newblock Forward-backward and {T}seng’s type penalty schemes for monotone
  inclusion problems.
\newblock \emph{Set-Valued and Variational Analysis}, 22\penalty0 (2):\penalty0
  313--331, 2014.
\newblock ISSN 1877-0533.
\newblock \doi{10.1007/s11228-014-0274-7}.
\newblock URL \url{http://dx.doi.org/10.1007/s11228-014-0274-7}.

\bibitem[Bo{\c{t}} et~al.(2013{\natexlab{a}})Bo{\c{t}}, Csetnek, and
  Heinrich]{botSIOPT}
R.~I. Bo{\c{t}}, E.~R. Csetnek, and A.~Heinrich.
\newblock A primal-dual splitting algorithm for finding zeros of sums of
  maximal monotone operators.
\newblock \emph{SIAM Journal on Optimization}, 23\penalty0 (4):\penalty0
  2011--2036, 2013{\natexlab{a}}.
\newblock \doi{10.1137/12088255X}.
\newblock URL \url{http://dx.doi.org/10.1137/12088255X}.

\bibitem[Bo{\c{t}} et~al.(2013{\natexlab{b}})Bo{\c{t}}, Csetnek, and
  Nagy]{boct2013solving}
R.~I. Bo{\c{t}}, E.~R. Csetnek, and E.~Nagy.
\newblock Solving systems of monotone inclusions via primal-dual splitting
  techniques.
\newblock \emph{Taiwanese Journal of Mathematics}, 17\penalty0 (6):\penalty0
  pp--1983, 2013{\natexlab{b}}.

\bibitem[Bo{\c{t}} et~al.(2015)Bo{\c{t}}, Csetnek, Heinrich, and
  Hendrich]{botMathProg}
R.~I. Bo{\c{t}}, E.~R. Csetnek, A.~Heinrich, and C.~Hendrich.
\newblock On the convergence rate improvement of a primal-dual splitting
  algorithm for solving monotone inclusion problems.
\newblock \emph{Mathematical Programming}, 150\penalty0 (2):\penalty0 251--279,
  2015.
\newblock ISSN 0025-5610.
\newblock \doi{10.1007/s10107-014-0766-0}.
\newblock URL \url{http://dx.doi.org/10.1007/s10107-014-0766-0}.

\bibitem[Bouwmans and Zahzah(2014)]{Bouwmans201422}
T.~Bouwmans and E.~H. Zahzah.
\newblock Robust {PCA} via principal component pursuit: A review for a
  comparative evaluation in video surveillance.
\newblock \emph{Computer Vision and Image Understanding}, 122\penalty0
  (0):\penalty0 22--34, 2014.
\newblock ISSN 1077-3142.
\newblock \doi{http://dx.doi.org/10.1016/j.cviu.2013.11.009}.
\newblock URL
  \url{http://www.sciencedirect.com/science/article/pii/S1077314213002294}.

\bibitem[Briceño-Arias and Combettes(2011)]{MonotoneSkew}
L.~M. Briceño-Arias and P.~L. Combettes.
\newblock A monotone+skew splitting model for composite monotone inclusions in
  duality.
\newblock \emph{SIAM J. Optim.}, 21\penalty0 (4):\penalty0 1230--1250, 2011.

\bibitem[Brucker(1984)]{knapsack1984}
P.~Brucker.
\newblock An {O}(n) algorithm for quadratic knapsack problems.
\newblock \emph{Operations Res. Lett.}, 3\penalty0 (3):\penalty0 163 -- 166,
  1984.
\newblock \doi{10.1016/0167-6377(84)90010-5}.

\bibitem[Bruer et~al.(2014)Bruer, Tropp, Cevher, and Becker]{NIPS_Bruer2014}
J.~J. Bruer, J.~A. Tropp, V.~Cevher, and S.~Becker.
\newblock Time--data tradeoffs by aggressive smoothing.
\newblock In \emph{Advances in Neural Information Processing Systems}, pages
  1664--1672, 2014.

\bibitem[Burke and Qian(1997)]{BurkeQian97}
J.~Burke and M.~Qian.
\newblock A variable metric proximal point algorithm for monotone operators.
\newblock \emph{SIAM J. Control Opt.}, 37\penalty0 (2):\penalty0 353--375,
  1997.

\bibitem[Cai et~al.(2010)Cai, Cand\`es, and Shen]{SVT}
J-F. Cai, E.~J. Cand\`es, and Z.~Shen.
\newblock A singular value thresholding algorithm for matrix completion.
\newblock \emph{SIAM J. Optim.}, 20:\penalty0 1956--1982, 2010.

\bibitem[Cand\'es and Plan(2010)]{candes2010matrix}
E.~J. Cand\'es and Y.~Plan.
\newblock Matrix completion with noise.
\newblock \emph{Proceedings of the IEEE}, 98\penalty0 (6):\penalty0 925--936,
  2010.

\bibitem[Cand\`{e}s et~al.(2011{\natexlab{a}})Cand\`{e}s, Li, Ma, and
  Wright]{CanLiMa:11}
E.~J. Cand\`{e}s, X.~Li, Y.~Ma, and J.~Wright.
\newblock Robust principal component analysis?
\newblock \emph{J. Assoc. Comput. Mach.}, 58\penalty0 (3):\penalty0 1--37, May
  2011{\natexlab{a}}.

\bibitem[Cand\`{e}s et~al.(2011{\natexlab{b}})Cand\`{e}s, Li, Ma, and
  Wright]{candes2011robust}
E.J. Cand\`{e}s, X.~Li, Y.~Ma, and J.~Wright.
\newblock Robust principal component analysis?
\newblock \emph{J. ACM}, 58\penalty0 (3):\penalty0 11:1--11:37, June
  2011{\natexlab{b}}.
\newblock ISSN 0004-5411.
\newblock \doi{10.1145/1970392.1970395}.
\newblock URL \url{http://doi.acm.org/10.1145/1970392.1970395}.

\bibitem[Cevher et~al.(2014)Cevher, Becker, and Schmidt]{SPmag}
V.~Cevher, S.~Becker, and M.~Schmidt.
\newblock Convex optimization for big data: Scalable, randomized, and parallel
  algorithms for big data analytics.
\newblock \emph{IEEE Sig. Proc. Mag.}, 31\penalty0 (5):\penalty0 32--43, 2014.

\bibitem[Chambolle and Dossal(2014)]{FISTA_converges}
A.~Chambolle and C.~Dossal.
\newblock On the convergence of the iterates of ``{FISTA}''.
\newblock Technical Report hal-01060130, HAL, 2014.
\newblock \url{https://hal.inria.fr/hal-01060130v3}.

\bibitem[Chambolle and Pock(2010)]{ChambollePock10}
A.~Chambolle and T.~Pock.
\newblock A first-order primal-dual algorithm for convex problems with
  applications to imaging.
\newblock \emph{J. Math. Imaging Vision}, 40\penalty0 (1):\penalty0 120--145,
  2010.

\bibitem[Chandrasekaran et~al.(2009)Chandrasekaran, Sanghavi, Parrilo, and
  Willsky]{ChaSanPar:09}
V.~Chandrasekaran, S.~Sanghavi, P.~A. Parrilo, and A.~S. Willsky.
\newblock Sparse and low-rank matrix decompositions.
\newblock In \emph{SYSID 2009}, Saint-Malo, France, July 2009.

\bibitem[Chandrasekaran et~al.(2012)Chandrasekaran, Parrilo, and
  Willsky]{CPW12}
V.~Chandrasekaran, P.~A. Parrilo, and A.~S. Willsky.
\newblock Latent variable graphical model selection via convex optimization.
\newblock \emph{Ann. Stat.}, 40\penalty0 (4):\penalty0 1935--2357, 2012.

\bibitem[Chen et~al.(2013)Chen, He, Ye, and Yuan]{chen2013direct}
C.~Chen, B.~He, Y.~Ye, and X.~Yuan.
\newblock The direct extension of admm for multi-block convex minimization
  problems is not necessarily convergent.
\newblock \emph{Optimization Online}, 2013.

\bibitem[Ciarlet(1989)]{Ciarlet89}
P.~G. Ciarlet.
\newblock \emph{Introduction to numerical linear algebra and optimisation}.
\newblock Cambridge University Press, 1989.

\bibitem[Combettes(2013)]{combettes2013systems}
P.~L. Combettes.
\newblock Systems of structured monotone inclusions: duality, algorithms, and
  applications.
\newblock \emph{SIAM J. Optimization}, 23\penalty0 (4):\penalty0 2420--2447,
  2013.

\bibitem[Combettes and Pesquet(2007)]{Combettes2007}
P.~L. Combettes and J.-C. Pesquet.
\newblock {A Douglas–Rachford Splitting Approach to Nonsmooth Convex
  Variational Signal Recovery}.
\newblock \emph{IEEE J. Sel. Topics Sig. Processing}, 1\penalty0 (4):\penalty0
  564--574, December 2007.

\bibitem[Combettes and Pesquet(2011)]{CombettesPesquetChapter}
P.~L. Combettes and J.-C. Pesquet.
\newblock Proximal splitting methods in signal processing.
\newblock In H.~H. Bauschke, R.~S. Burachik, P.~L. Combettes, V.~Elser, D.~R.
  Luke, and H.~Wolkowicz, editors, \emph{Fixed-Point Algorithms for Inverse
  Problems in Science and Engineering}, pages 185--212. Springer-Verlag, New
  York, 2011.

\bibitem[Combettes and Pesquet(2012)]{CombettesPesquet12}
P.~L. Combettes and J.-C. Pesquet.
\newblock Primal-dual splitting algorithm for solving inclusions with mixtures
  of composite, {L}ipschitzian, and parallel-sum type monotone operators.
\newblock \emph{Set-Valued and Variational Analysis}, 20\penalty0 (2):\penalty0
  307--330, 2012.

\bibitem[Combettes and Wajs(2005)]{ComWaj:05}
P.~L. Combettes and V.~R. Wajs.
\newblock Signal recovery by proximal forward-backward splitting.
\newblock \emph{Multiscale Model. Simul.}, 4\penalty0 (4):\penalty0 1168--1200
  (electronic), 2005.
\newblock ISSN 1540-3459.
\newblock \doi{10.1137/050626090}.
\newblock URL \url{http://dx.doi.org/10.1137/050626090}.

\bibitem[Combettes et~al.(2010)Combettes, D{\~u}ng, and
  V{\~u}]{combettesDualization}
P.~L. Combettes, D.~D{\~u}ng, and B.~C. V{\~u}.
\newblock {Dualization of signal recovery problems}.
\newblock \emph{Set-Valued and Variational Analysis}, 18:\penalty0 373--404,
  2010.
\newblock ISSN 1877-0533.

\bibitem[Condat(2013)]{Condat2011}
L.~Condat.
\newblock A primal-dual splitting method for convex optimization involving
  {L}ipschitzian, proximable and linear composite terms.
\newblock \emph{J. Optim. Theory Appl.}, pages 460--479, 2013.

\bibitem[Devolder et~al.(2011)Devolder, Glineur, and Nesterov]{Devolder2011}
O.~Devolder, F.~Glineur, and Y.~Nesterov.
\newblock First-order methods of smooth convex optimization with inexact
  oracle.
\newblock \emph{Math. Prog. {\em submitted}}, 2011.

\bibitem[Devolder et~al.(2012)Devolder, Glineur, and Nesterov]{DoubleSmoothing}
O.~Devolder, F.~Glineur, and Y.~Nesterov.
\newblock Double smoothing technique for large-scale linearly constrained
  convex optimization.
\newblock \emph{SIAM Journal on Optimization}, 22\penalty0 (2):\penalty0
  702--727, 2012.
\newblock \doi{10.1137/110826102}.
\newblock URL \url{http://dx.doi.org/10.1137/110826102}.

\bibitem[Duchi et~al.(2008)Duchi, Shalev-Shwartz, Singer, and
  Chandra]{Duchi2008}
J.~Duchi, S.~Shalev-Shwartz, Y.~Singer, and T.~Chandra.
\newblock {Efficient projections onto the l1-ball for learning in high
  dimensions}.
\newblock In \emph{Intl. Conf. Machine Learning (ICML)}, pages 272--279, New
  York, July 2008. ACM Press.

\bibitem[Esser et~al.(2009)Esser, Zhang, and Chan]{preconditionedADMM}
E.~Esser, X.~Zhang, and T.~Chan.
\newblock A general framework for a class of first order primal-dual algorithms
  for tv minimization.
\newblock Technical Report 09-67, UCLA, Center for Applied Math, 2009.

\bibitem[Friedlander and Tseng(2007)]{FrieTsen:2007}
M.~P. Friedlander and P.~Tseng.
\newblock Exact regularization of convex programs.
\newblock \emph{SIAM J. Optim.}, 18\penalty0 (4):\penalty0 1326--1350, 2007.
\newblock \doi{10.1137/060675320}.
\newblock URL \url{http://link.aip.org/link/?SJE/18/1326/1}.

\bibitem[Ganesh et~al.(2009)Ganesh, Lin, Wright, Wu, Chen, and Ma]{GaneshRPCA}
A.~Ganesh, Z.~Lin, J.~Wright, L.~Wu, M.~Chen, and Y.~Ma.
\newblock Fast algorithms for recovering a corrupted low-rank matrix.
\newblock In \emph{Computational Advances in Multi-Sensor Adaptive Processing
  (CAMSAP)}, pages 213--215, Aruba, Dec. 2009.

\bibitem[Goldfarb et~al.(2013)Goldfarb, Ma, and Scheinberg]{ShiqianFastADM}
D.~Goldfarb, S.~Ma, and K.~Scheinberg.
\newblock Fast alternating linearization methods for minimizing the sum of two
  convex functions.
\newblock \emph{Math. Prog. (Series A)}, 141\penalty0 (1--2):\penalty0
  349--382, 2013.

\bibitem[Goldstein et~al.(2013)Goldstein, Esser, and Baraniuk]{pdhg_2013}
T.~Goldstein, E.~Esser, and R.~Baraniuk.
\newblock Adaptive primal-dual hybrid gradient methods for saddle-point
  problems.
\newblock Technical report, 2013.
\newblock \url{http://arxiv.org/abs/1305.0546}.

\bibitem[Halko et~al.(2011)Halko, Martinsson, and Tropp]{halko2011finding}
N.~Halko, P.-G. Martinsson, and J.~A. Tropp.
\newblock Finding structure with randomness: Probabilistic algorithms for
  constructing approximate matrix decompositions.
\newblock \emph{SIAM review}, 53\penalty0 (2):\penalty0 217--288, 2011.

\bibitem[He and Yuan(2012{\natexlab{a}})]{HeYuan10}
B.S. He and X.~M. Yuan.
\newblock Convergence analysis of primal-dual algorithms for a saddle-point
  problem: from contraction perspective.
\newblock \emph{SIAM J. Imaging Sci.}, pages 119--149, 2012{\natexlab{a}}.

\bibitem[He and Yuan(2012{\natexlab{b}})]{HeYuanADMMrate}
B.S. He and X.M. Yuan.
\newblock On the o(1/n) convergence rate of the {D}ouglas-{R}achford
  alternating direction method.
\newblock \emph{SIAM J. Numer. Anal.}, 50\penalty0 (700--709),
  2012{\natexlab{b}}.

\bibitem[Huber(2004)]{Hub}
P.~J. Huber.
\newblock \emph{Robust Statistics}.
\newblock John Wiley and Sons, 2 edition, 2004.

\bibitem[Komodakis and Pesquet(2015)]{Duality2015}
N.~Komodakis and J.-C. Pesquet.
\newblock Playing with duality: An overview of recent primal-dual approaches
  for solving large-scale optimization problems.
\newblock \emph{IEEE Sig. Pro. Magazine, \emph{to appear}}, May 2015.
\newblock \url{http://arxiv.org/abs/1406.5429v2}.

\bibitem[Lai and Yin(2013)]{YinNuclearPenalty}
M.~Lai and W.~Yin.
\newblock Augmented l1 and nuclear-norm models with a globally linearly
  convergent algorithm.
\newblock \emph{SIAM J. Imaging Sci.}, 6\penalty0 (2):\penalty0 1059--1091,
  2013.

\bibitem[Lee et~al.(2010)Lee, Recht, Salakhutdinov, Srebro, and
  Tropp]{JasonLee}
J.~Lee, B.~Recht, R.~Salakhutdinov, N.~Srebro, and J.A. Tropp.
\newblock Practical large-scale optimization for max-norm regularization.
\newblock In \emph{Neural Information Processing Systems (NIPS)}, Vancouver,
  2010.

\bibitem[Li et~al.(2014)Li, Mo, Yuan, and Zhang]{LinADMM_LiEtAl}
X.~X. Li, L.~L. Mo, X.~M. Yuan, and J.~Z. Zhang.
\newblock Linearized alternating direction method of multipliers for sparse
  group and fused {LASSO} models.
\newblock \emph{Comput. Statis. Data Anal.}, 79:\penalty0 203--221, 2014.

\bibitem[Lin et~al.(2010)Lin, Chen, and Ma]{lin2010augmented}
Z.~Lin, M.~Chen, and Y.~Ma.
\newblock The augmented {L}agrange multiplier method for exact recovery of
  corrupted low-rank matrices.
\newblock \emph{arXiv preprint arXiv:1009.5055}, 2010.

\bibitem[Liu et~al.(2011)Liu, Sun, and Toh]{LiuToh09}
Y.-J. Liu, D.~Sun, and K.-C. Toh.
\newblock An implementable proximal point algorithmic framework for nuclear
  norm minimization.
\newblock \emph{Mathematical Programming}, 2011.

\bibitem[M.~Schmidt(2011)]{SchmidtInexact}
F.~Bach M.~Schmidt, N. Le~Roux.
\newblock Convergence rates of inexact proximal-gradient methods for convex
  optimization.
\newblock In \emph{NIPS}, 2011.

\bibitem[Malgouyres and Zeng(2009)]{PPPA}
F.~Malgouyres and T.~Zeng.
\newblock A predual proximal point algorithm solving a non negative basis
  pursuit denoising model.
\newblock \emph{Int. J. Comp. Vision}, 83\penalty0 (3):\penalty0 294--311, July
  2009.

\bibitem[Mangasarian and Meyer(1979)]{Mangasarian79}
O.~L. Mangasarian and R.~R. Meyer.
\newblock Nonlinear perturbation of linear programs.
\newblock \emph{SIAM J. Control Optim.}, 17:\penalty0 745--752, 1979.

\bibitem[Moreau(1965)]{Moreau1965}
J.-J. Moreau.
\newblock Proximit\'e et dualit\'e dans un espace hilbertien.
\newblock \emph{Bull. Soc. Math. France}, 93:\penalty0 273--299, 1965.

\bibitem[Necoara and Suykens(2008)]{Necoura08}
I.~Necoara and J.~A.~K. Suykens.
\newblock Application of a smoothing technique to decomposition in convex
  optimization.
\newblock \emph{IEEE Trans. Auto. Control}, 53\penalty0 (11):\penalty0
  2674--2679, 2008.

\bibitem[Nesterov(1983)]{Nesterov83}
Y.~Nesterov.
\newblock A method of solving a convex programming problem with convergence
  rate $\mathcal{O}(1/k^2)$, in soviet mathematics doklady.
\newblock \emph{Soviet Mathematics Doklady}, 27, 1983.

\bibitem[Nesterov(2004)]{Nesterov2004}
Y.~Nesterov.
\newblock \emph{Introductory lectures on convex optimization: a basic course},
  volume~87 of \emph{Applied Optimization}.
\newblock Kluwer Academic Publishers, 2004.

\bibitem[Nesterov(2005)]{Nesterov05}
Y.~Nesterov.
\newblock Smooth minimization of non-smooth functions.
\newblock \emph{Math. Program., Ser. A}, 103:\penalty0 127--152, 2005.

\bibitem[Nocedal and Wright(2006)]{NocedalWright}
J.~Nocedal and S.~Wright.
\newblock \emph{Numerical Optimization}.
\newblock Springer, 2 edition, 2006.

\bibitem[Peng et~al.(2012)Peng, Ganesh, Wright, Xu, and Ma]{PenGanWri:12}
Y.~Peng, A.~Ganesh, J.~Wright, W.~Xu, and Y.~Ma.
\newblock {RASL}: Robust alignment by sparse and low-rank decomposition for
  linearly correlated images.
\newblock \emph{IEEE Trans. Pattern Analysis and Machine Intelligence},
  34\penalty0 (11):\penalty0 2233--2246, 2012.
\newblock ISSN 0162-8828.
\newblock \doi{http://doi.ieeecomputersociety.org/10.1109/TPAMI.2011.282}.

\bibitem[Pock and Chambolle(2011)]{PockPrecond}
T.~Pock and A.~Chambolle.
\newblock Diagonal preconditioning for first order primal-dual algorithms in
  convex optimization.
\newblock In \emph{Computer Vision (ICCV), 2011 IEEE International Conference
  on}, pages 1762--1769, Nov 2011.
\newblock \doi{10.1109/ICCV.2011.6126441}.

\bibitem[Poljak and Tretjakov(1974)]{Poljak74}
B.~T. Poljak and N.~V. Tretjakov.
\newblock An iterative method for linear programming and its economic
  interpretation.
\newblock \emph{Matecon}, 10:\penalty0 81--100, 1974.

\bibitem[Ren and Lin(2013)]{LinADMM_RenEtAl}
X.~Ren and Z.~Lin.
\newblock Linearized alternating direction method with adaptive penalty and
  warm starts for fast solving transform invariant low-rank textures.
\newblock \emph{Int. J. Comput. Vis.}, 104:\penalty0 1--14, 2013.

\bibitem[Rockafellar(1970{\natexlab{a}})]{Roc:70}
R.~T. Rockafellar.
\newblock \emph{Convex analysis}.
\newblock Princeton Mathematical Series, No. 28. Princeton University Press,
  Princeton, N.J., 1970{\natexlab{a}}.

\bibitem[Rockafellar(1976)]{Rockafellar1976}
R.~T. Rockafellar.
\newblock Monotone operators and the proximal point algorithm.
\newblock \emph{SIAM Journal on Control and Optimization}, 14:\penalty0
  877--898, 1976.

\bibitem[Rockafellar(1970{\natexlab{b}})]{RTR}
R.~Tyrrell Rockafellar.
\newblock \emph{Convex Analysis}.
\newblock Priceton Landmarks in Mathematics. Princeton University Press,
  1970{\natexlab{b}}.

\bibitem[Rockafellar and Wets(1998)]{RTRW}
R.~Tyrrell Rockafellar and Roger J-B. Wets.
\newblock \emph{Variational Analysis}, volume 317 of \emph{A Series of
  Comprehensive Studies in Mathematics}.
\newblock Springer, 1998.

\bibitem[Schmidt et~al.(2009)Schmidt, van~den Berg, Friedlander, and
  Murphy]{PQN}
M.~Schmidt, E.~van~den Berg, M.~Friedlander, and K.~Murphy.
\newblock Optimizing costly functions with simple constraints: A limited-memory
  projected quasi-{N}ewton algorithm.
\newblock In \emph{AISTATS}, 2009.

\bibitem[Shen et~al.(2014)Shen, Wen, and Zhang]{Shen2014}
Y.~Shen, Z.~Wen, and Y.~Zhang.
\newblock {Augmented Lagrangian alternating direction method for matrix
  separation based on low-rank factorization}.
\newblock \emph{Optimization Methods and Software}, 29\penalty0 (2):\penalty0
  239--263, March 2014.

\bibitem[Shi et~al.(2013)Shi, Han, Zheng, and Li]{ExactPenaltyTensor}
Z.~Shi, J.~Han, T.~Zheng, and J.~Li.
\newblock Guarantees of augmented trace norm models in tensor recovery.
\newblock In \emph{Int. Joint Conf. Artificial Intell.}, pages 1670--1676,
  2013.

\bibitem[Tao and Yuan(2011)]{ASALM}
M.~Tao and X.~Yuan.
\newblock Recovering low-rank and sparse components of matrices from incomplete
  and noisy observations.
\newblock \emph{SIAM J. Optimization}, 21:\penalty0 57--81, 2011.

\bibitem[Tran-Dinh and Cevher(2014)]{NIPS2014_5494_Volkan}
Q.~Tran-Dinh and V.~Cevher.
\newblock Constrained convex minimization via model-based excessive gap.
\newblock In Z.~Ghahramani, M.~Welling, C.~Cortes, N.D. Lawrence, and K.Q.
  Weinberger, editors, \emph{Advances in Neural Information Processing Systems
  27}, pages 721--729. Curran Associates, Inc., 2014.
\newblock URL
  \url{http://papers.nips.cc/paper/5494-constrained-convex-minimization-via-model-based-excessive-gap.pdf}.

\bibitem[van~den Berg and Friedlander(2008)]{SPGL}
E.~van~den Berg and M.~P. Friedlander.
\newblock Probing the {P}areto frontier for basis pursuit solutions.
\newblock \emph{SIAM J. Sci. Comput.}, 31\penalty0 (2):\penalty0 890--912,
  2008.

\bibitem[Van~den berg and Friedlander(2008)]{vandenberg2008probing}
E.~Van~den berg and M.~P. Friedlander.
\newblock Probing the {P}areto frontier for basis pursuit solutions.
\newblock \emph{SIAM J. Sci. Computing}, 31\penalty0 (2):\penalty0 890--912,
  2008.
\newblock software: \url{http://www.cs.ubc.ca/~mpf/spgl1/}.

\bibitem[van~den Berg and Friedlander(2011)]{BergFriedlander:2011}
E.~van~den Berg and M.~P. Friedlander.
\newblock Sparse optimization with least-squares constraints.
\newblock \emph{SIAM J. Optimization}, 21\penalty0 (4):\penalty0 1201--1229,
  2011.

\bibitem[Villa et~al.(2013)Villa, Salzo, Baldassare, and Verri]{VillaInexact}
S.~Villa, S.~Salzo, L.~Baldassare, and A.~Verri.
\newblock Accelerated and inexact forward-backward algorithms.
\newblock \emph{SIAM J. Optimization}, 23\penalty0 (3):\penalty0 1607--1633,
  2013.
\newblock \doi{10.1137/110844805}.

\bibitem[V{\~u}(2013)]{Vu2011}
B.~C. V{\~u}.
\newblock A splitting algorithm for dual monotone inclusions involving
  cocoercive operators.
\newblock \emph{Adv. Comput. Math}, 38\penalty0 (3):\penalty0 667--681, 2013.

\bibitem[Wang and Yuan(2012)]{LinADMM_WangYuan}
X.~Wang and X.~M. Yuan.
\newblock The linearized alternating direction method of multipliers for
  {D}antzig selector.
\newblock \emph{SIAM J. Sci. Comput.}, 34:\penalty0 2782--2811, 2012.

\bibitem[Wright et~al.(2009{\natexlab{a}})Wright, Ganesh, Rao, and
  Ma]{RPCA_algo_Wright}
J.~Wright, A.~Ganesh, S.~Rao, and Y.~Ma.
\newblock Robust principal component analysis: Exact recovery of corrupted
  low-rank matrices by convex optimization.
\newblock In \emph{Neural Information Processing Systems (NIPS)},
  2009{\natexlab{a}}.

\bibitem[Wright et~al.(2009{\natexlab{b}})Wright, Nowak, and
  Figueiredo]{Wright2009}
S.~J. Wright, R.~D. Nowak, and M.~A.~T. Figueiredo.
\newblock Sparse reconstruction by separable approximation.
\newblock \emph{IEEE Trans. Sig. Processing}, 57\penalty0 (7):\penalty0
  2479--2493, July 2009{\natexlab{b}}.

\bibitem[Yang and Zhang(2011)]{YALL1}
J.~Yang and Y.~Zhang.
\newblock Alternating direction algorithms for l1-problems in compressive
  sensing.
\newblock \emph{SIAM J. Sci. Comput.}, 33\penalty0 (1--2):\penalty0 250--278,
  2011.

\bibitem[Yang and Yuan(2013)]{LinADMM_YangYuan}
J.~F. Yang and X.~M. Yuan.
\newblock Linearized augmented {L}agrangian and alternating direction methods
  for nuclear norm minimization.
\newblock \emph{Math. Comput.}, 82:\penalty0 301--329, 2013.

\bibitem[Yin(2010)]{YinPenalty}
W.~Yin.
\newblock Analysis and generalizations of the linearized {Bregman} method.
\newblock \emph{SIAM J. Imaging Sci.}, 3\penalty0 (4):\penalty0 856--877, 2010.
\newblock \doi{10.1137/090760350}.
\newblock URL \url{http://link.aip.org/link/?SII/3/856/1}.

\bibitem[You et~al.(2013)You, Wan, and Liu]{ExactPenaltyRPCA_fix}
Q.~You, Q.~Wan, and Y.~Liu.
\newblock A short note on strongly convex programming for exact matrix
  completion and robust principal component analysis.
\newblock \emph{Inverse Problems and Imaging}, 7\penalty0 (1):\penalty0
  305--306, 2013.

\bibitem[Zhang et~al.(2012)Zhang, Cai, Cheng, and Zhu]{ExactPenaltyRPCA}
J.~Zhang, J.-F. Cai, L.~Cheng, and J.~Zhu.
\newblock Strongly convex programming for exact matrix completion and robust
  principal component analysis.
\newblock \emph{Inverse Problems and Imaging}, 6\penalty0 (2):\penalty0
  357--372, 2012.

\bibitem[Zhang et~al.(2010)Zhang, Burger, and Osher]{LinADMM_ZhangBurgerOsher}
X.~Q. Zhang, M.~Burger, and S.~Osher.
\newblock A unified primal-dual algorithm framework based on {B}regman
  iteration.
\newblock \emph{J. Sci. Comput.}, 46:\penalty0 20--46, 2010.

\bibitem[Zhang et~al.(2011)Zhang, Liang, Ganesh, and Ma]{ZhaLiaGan:11}
Z.~Zhang, X.~Liang, A.~Ganesh, and Y.~Ma.
\newblock {TILT}: Transform invariant low-rank textures.
\newblock In R.~Kimmel, R.~Klette, and A.~Sugimoto, editors, \emph{Computer
  Vision -- ACCV 2010}, volume 6494 of \emph{Lecture Notes in Computer
  Science}, pages 314--328. Springer, 2011.

\bibitem[Zhao et~al.(2010)Zhao, Sun, and Toh]{SDPNAL}
X.-Y. Zhao, D.~Sun, and K.-C. Toh.
\newblock A {Newton-CG} augmented {L}agrangian method for semidefinite
  programming.
\newblock \emph{SIAM Journal on Optimization}, 20\penalty0 (4):\penalty0
  1737--1765, 2010.
\newblock \doi{10.1137/080718206}.
\newblock URL \url{http://dx.doi.org/10.1137/080718206}.

\end{thebibliography}

\end{document}